\documentclass[preprint,12p,number]{elsarticle}
\usepackage[dvipsnames, table]{xcolor}
\usepackage[breaklinks=true,bookmarks=true]{hyperref}

\usepackage{bigints}
\usepackage{graphicx}%
\usepackage{multirow}%
\usepackage{amsmath,amssymb,amsfonts}%
\usepackage{amsthm}%
\usepackage{mathrsfs}%
\usepackage[title]{appendix}%
\usepackage{xcolor}%
\usepackage{textcomp}%
\usepackage{manyfoot}%
\usepackage{booktabs}%
\usepackage{algorithm}%
\usepackage{algorithmicx}%
\usepackage{algpseudocode}%
\usepackage{listings}%
\usepackage[table]{xcolor}
\usepackage{hhline}
\usepackage{array}
\usepackage{subcaption}
\usepackage{dashrule}
\usepackage{tikz}
\usepackage{pgfplots}
\usepackage{array}
\usepackage{subcaption}
\usepackage{dashrule}

\newcolumntype{C}[1]{>{\centering\arraybackslash}m{#1}}

\theoremstyle{thmstyleone}%
\theoremstyle{thmstyletwo}%
\newtheorem{remark}{Remark}%

\theoremstyle{thmstylethree}%
\newtheorem{lemma}{Lemma}%

\DeclareMathOperator*{\argmin}{arg\,min}

\usepackage{pifont}

\usepackage{nicematrix}
\definecolor{lgreen}{RGB}{236, 255, 201}
\definecolor{lightblue}{rgb}{0.678, 0.847, 0.902}

\begin{document}

\begin{frontmatter}

\title{Robust Basis Spline Decoupling for the Compression of Transformer Models}

\author[1]{Joppe De Jonghe\corref{cor1}\fnref{label2}}
\ead{joppe.dejonghe@kuleuven.be}

\author[2]{Van Tien Pham\fnref{label2}}

\author[1]{Mariya Ishteva}

\affiliation[1]{
            organization={NUMA, Department of Computer Science, KU Leuven},
            addressline={Kleinhoefstraat 4},
            city={Geel},
            postcode={2440},
            state={Antwerpen},
            country={België}
            }

\affiliation[2]{
            organization={Université Paris-Saclay, CEA, List},
            city={Palaiseau},
            postcode={F-91120},
            country={France}
            }

\fntext[label2]{These authors contributed equally to this work as co-first authors}
\cortext[cor1]{Corresponding author}

\begin{abstract}
Decoupling is a powerful modeling paradigm for representing multivariate functions as compositions of linear transformations and univariate nonlinear functions. A single-layer decoupling can be viewed as a fully connected neural network with a single hidden layer and flexible activation functions, providing a direct link with neural networks. Because of this, the use of decoupling methods has gained increasing attention in neural network domains, particularly compression, since it enables structured approximations with reduced parameter complexity. Existing tensor‑based decoupling methods typically rely on polynomial or piecewise‑linear parameterizations of the internal nonlinear functions, which can suffer from numerical instability or limited expressiveness.

In this work, we introduce a B-spline-based decoupling framework that generalizes these existing approaches. By exploiting the local support and flexible smoothness control of B-splines, the proposed formulation yields a more numerically stable and expressive representation. We derive a constrained coupled matrix–tensor factorization and propose a robust alternating least‑squares algorithm, called R-CMTF-BSD, incorporating normalization and Tikhonov regularization. The proposed method is validated through experiments on synthetic data and transformer model compression. Results on the Vision and Swin Transformer architectures demonstrate that B‑spline decoupling enables substantial parameter reduction while maintaining competitive accuracy, making the R-CMTF-BSD algorithm a promising tool for structured neural network compression.
\end{abstract}


\end{frontmatter}

\section{Introduction}\label{sec1}
Neural networks are widely used to model high‑dimensional nonlinear functions across domains such as vision, language, and system identification. Despite their expressive power, modern architectures often rely on heavily overparameterized components, limiting interpretability and increasing computational and memory costs \cite{pham2024efficient}. These concerns motivate research into structured representations that retain expressive power while enabling model compression, analysis, and efficient deployment.

A promising approach in this direction is the decoupling paradigm introduced by Dreesen et al. \cite{dreesen2015decoupling}, which represents a multivariate vector‑valued function as a linear mapping followed by a layer of univariate nonlinear functions and a final linear transformation (see Figure \ref{fig:decoupling}). The authors show that under mild conditions, such representations are identifiable from first‑order information by evaluating the Jacobian at multiple inputs and exploiting the resulting tensor structure through the canonical polyadic decomposition (CPD). However, the approach in \cite{dreesen2015decoupling} relies on the uniqueness of the tensor decomposition, which doesn't hold in noisy or approximate settings. To this end, Hollander \cite{hollander2017multivariate} proposed a polynomial constraint algorithm. Zniyed et al. \cite{zniyed2021tensor} generalized this idea by formulating a general basis function representation of the internal functions and incorporating zeroth‑order information through coupled matrix–tensor factorizations. Other extensions and adaptations of decoupling algorithms exist; some examples include~\cite{decuyper2022decoupling, de2026tensor, dreesen2018decoupling}.

The decoupled representation represents a fully connected neural network with a single hidden layer and flexible activation functions. This connection between decoupling and neural networks has been exploited in \cite{westwick2021decoupling}, where they applied it to a neural network NARX model and \cite{zniyed2021tensor}, which focused explicitly on the compression of a convolutional neural network.

However, existing decoupling methods typically rely on restrictive bases for the internal functions. Polynomial decoupling \cite{dreesen2015decoupling, hollander2017multivariate} employs monomial bases, which suffer from numerical ill‑conditioning and poor extrapolation behaviour for higher degrees. The approach in \cite{zniyed2021tensor} replaces the polynomial bases with piecewise‑linear representations constructed from shifted ReLU (rectified linear unit) functions, improving numerical behaviour while remaining closely tied to neural network architectures. Nevertheless, this imposes structural constraints on the function space, limiting expressiveness and smoothness.

These limitations motivate the need for a more general and numerically stable basis for decoupling. In this work, we propose a B‑spline‑based parameterization of the internal functions in decoupled models. B‑splines construct locally supported bases with controllable smoothness and approximation power determined by the spline degree and knot placement. Importantly, polynomial and piecewise‑linear bases arise as special cases of the B‑spline formulation under appropriate configurations, establishing B‑splines as a strict generalization of existing approaches. This unification allows decoupled models to easily move between polynomial representations, piecewise-linear activations, and more expressive spline‑based models within a single framework.
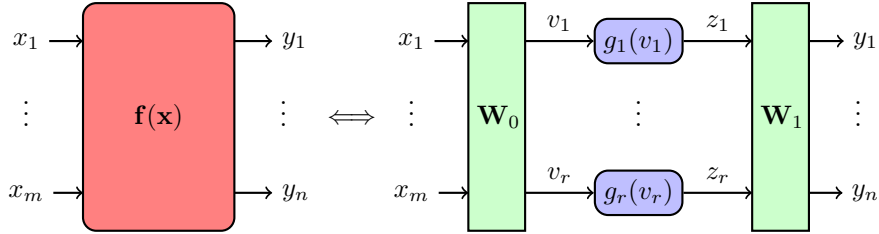
\begin{figure*}[!htbp]
\begin{center}
\begin{tikzpicture}
\node (x1) at (0,3) {$x_1$};
\node at (0,2.15) {$\vdots$};
\node (xm) at (0,1) {$x_m$};
\draw [thick,fill=red!50, rounded corners=5pt] (0.75,0.5) rectangle (2.75,3.5); \node (F) at (1.75,2) {$\mathbf{f}(\mathbf{x})$};
\draw [->, thick, label=] (2.75,3) -- (3.25,3) node[right] {$y_1$};
\node at (3.4,2.15) {$\vdots$};
\draw [->, thick] (2.75,1) -- (3.25,1) node[right] {$y_n$};
\draw [->, thick] (x1) -- (0.75,3);
\draw [->, thick] (xm) -- (0.75,1);
\end{tikzpicture}
\raisebox{11\height}{$\Longleftrightarrow$}
\begin{tikzpicture}
\node (x1) at (0,3) {$x_1$};
\node at (0,2.15) {$\vdots$};
\node (xm) at (0,1) {$x_m$};
\draw [thick, fill=green!20] (0.75,0.5) rectangle (1.5,3.5); 
\node (L) at (1.15,2) {$\mathbf{W}_0$};
\draw [->, thick] (x1) -- (0.75,3);
\draw [->, thick] (xm) -- (0.75,1);
\node [shape=rectangle,draw,thick,fill=blue!25,rounded corners=5pt] (g1) at (3,3) {$g_1(v_1)$};
\draw [->, thick, label=] (1.5,3) -- (g1) node[above,midway] {$v_1$};
\node at (3,2.15) {$\vdots$};
\node [shape=rectangle,draw,thick,fill=blue!25,rounded corners=5pt] (gr) at (3,1) {$g_r(v_r)$};
\draw [->, thick] (1.5,1) -- (gr) node[above,midway] {$v_r$};
\draw [thick, fill=green!20] (4.5,0.5) rectangle (5.25,3.5); 
\node (R) at (4.9,2) {$\mathbf{W}_1$};
\draw [->, thick] (g1) -- (4.5,3) node[above,midway] {$z_1$};
\draw [->, thick] (gr) -- (4.5,1) node[above,midway] {$z_r$};
\node (y1) at (6,3) {$y_1$};
\node at (5.9,2.15) {$\vdots$};
\node (yn) at (6,1) {$y_n$};
\draw [->, thick] (5.25,3) -- (y1);
\draw [->, thick] (5.25,1) -- (yn);
\end{tikzpicture}
\end{center}
\caption{The single-layer decoupling problem. Given a multivariate vector function $\mathbf{f}: \mathbb{R}^{m} \rightarrow \mathbb{R}^{n}$, determine factor matrices $\mathbf{W}_0 \in \mathbb{R}^{r \times m}$, $\mathbf{W}_1 \in \mathbb{R}^{n \times r}$ and the internal functions $\mathbf{g} = \begin{bmatrix}g_{1} & g_{2} & \cdots & g_r\end{bmatrix}^\top$, $g_i: \mathbb{R} \rightarrow \mathbb{R}$, so that
    $\mathbf{f}(\mathbf{x}) = \mathbf{W}_1
    \mathbf{g}(\mathbf{W}_0 \mathbf{x})$.}
    \label{fig:decoupling}
\end{figure*}
In this work, we develop a B-spline constrained CMTF formulation for single-layer decoupling. To address numerical instabilities observed in practice, particularly in neural network compression settings, we introduce a robust alternating least‑squares algorithm that incorporates normalization and Tikhonov regularization. The resulting method, named R-CMTF-BSD, enforces the decoupled structure through projection steps while maintaining stability in practical settings.

The main contributions of this work are: 1) A unifying B‑spline decoupling framework that generalizes polynomial and piecewise‑linear decoupling methods, 2) a spline‑constrained CMTF formulation for single‑layer decoupling that incorporates both first- and zeroth-order information, 3) a robust alternating least‑squares algorithm (R‑CMTF‑BSD) with normalization and Tikhonov regularization to mitigate numerical instabilities, 4) an application to transformer compression, demonstrating effective parameter reduction with limited accuracy loss and 5) an evaluation of compression strategies referred to as \textit{back-to-front} and \textit{front-to-back}, highlighting the advantages of \textit{back‑to‑front} decoupling for full network compression.

The remainder of this paper is organized as follows. Section \ref{sec:background} introduces notation and background on tensor decompositions and B‑spline bases. Section \ref{sec:single-layer-decoupling} reviews single‑layer decoupling and its tensor‑based formulation. Section \ref{sec:B-spline_parameterization} presents the B‑spline parameterization and establishes its relation to polynomial and piecewise‑linear decoupling. Section \ref{sec:opt_algorithm} introduces the constrained optimization problem and the proposed robust algorithm. Section \ref{sec:transformer_compression} describes transformer models and the application of decoupling for compression, together with possible compression strategies. Section \ref{sec:experiments} reports experimental results on synthetic data and transformer architectures. Finally, Section \ref{sec:conclusion} concludes the paper and outlines future work.

\section{Notation and Background}\label{sec:background}

\subsection{Notation}

Vectors, matrices, and tensors are denoted by bold lowercase, bold capital, and calligraphic letters respectively, e.g., the vector $\mathbf{a}$, matrix $\mathbf{A}$, and tensor $\mathcal{X}$. For a matrix $\mathbf{A}$, the $i$-th row and $j$-th column are denoted as $\mathbf{A}^{i,:}$ and $\mathbf{A}^{:,j}$ respectively. For a third order tensor $\mathcal{X}$ of size $I \times J \times K$, the $i$-th horizontal, $j$-th lateral and $k$-th frontal slice are denoted by $\mathcal{X}_{i,:,:}$, $\mathcal{X}_{:,j,:}$ and $\mathcal{X}_{:,:,k}$ respectively. The operation $\operatorname{unfold}_k(\mathcal{X})$ unfolds the tensor $\mathcal{X}$ over its $k$-th mode as described in \cite{kolda2009tensor}. The symbol $\odot$ denotes the Khatri-Rao product. Finally, the $\operatorname{diag}(.)$ operation forms a diagonal matrix where the main diagonal is the vector that is provided as a parameter and $\lVert . \rVert$ denotes the norm of a tensor, defined as the square root of the sum of the squares of its elements, so for a tensor $\mathcal{X} \in \mathbb{R}^{I \times J \times K}$
\begin{equation}
    \lVert \mathcal{X} \rVert = \sqrt{\sum^I_{i=1}\sum^J_{j=1}\sum^K_{k=1}x^2_{ijk}}. \label{eq:CPD}
\end{equation}

\subsection{Canonical polyadic decomposition}

A third-order tensor $\mathcal{X} \in \mathbb{R}^{I \times J \times K}$ admits a \textit{canonical polyadic decomposition} (CPD) \cite{kolda2009tensor} if it can be written as a sum of rank one tensors
\begin{align}
    \mathcal{X} = \sum^r_{i=1} \mathbf{a}_i \circ \mathbf{b}_i \circ \mathbf{c}_i = [\![\mathbf{A}, \mathbf{B}, \mathbf{C}]\!], \nonumber
\end{align}
where the vectors $\mathbf{a}_i \in \mathbb{R}^{I}$, $\mathbf{b}_i \in \mathbb{R}^{J}$ and $\mathbf{c}_i \in \mathbb{R}^{K}$, for $i=1,\hdots,r$, are collected as columns in the factor matrices $\mathbf{A} \in \mathbb{R}^{I \times r}$, $\mathbf{B} \in \mathbb{R}^{J \times r}$ and $\mathbf{C}\in\mathbb{R}^{K \times r}$. The canonical rank $r$ is the lowest value for which equation \eqref{eq:CPD} holds. Alternatively, it can be stated that the tensor $\mathcal{X}$ admits a CPD if its frontal slices can be represented as 
\begin{equation}
    \mathcal{X}_{:,:,k} = \mathbf{A} \cdot \text{diag}(\mathbf{C}^{k,:}) \cdot \mathbf{B}, \text{ for } k=1,2,\hdots,K. \label{eq:CPD_diag}
\end{equation} 

\noindent The CPD described in equation \eqref{eq:CPD} is unique under mild conditions. Uniqueness means that the CPD is unique up to the following scaling and permutation ambiguities \cite{kolda2009tensor}:
\begin{equation}
    \mathcal{X} = [\![\mathbf{A}\mathbf{\Pi}\mathbf{\Lambda_{A}}, \mathbf{B}\mathbf{\Pi}\mathbf{\Lambda_{B}}, \mathbf{C}\mathbf{\Pi}\mathbf{\Lambda_{C}}]\!], \nonumber
\end{equation}
with permutation matrix $\mathbf{\Pi} \in \mathbb{R}^{r \times r}$ and diagonal matrices $\mathbf{\Lambda_{A}}$, $\mathbf{\Lambda_{B}}$, $\mathbf{\Lambda_{C}}$ for which $\mathbf{\Lambda_{A}}\mathbf{\Lambda_{B}}\mathbf{\Lambda_{C}}=\mathbf{I}$.
Several sufficient uniqueness conditions exist (for example, Kruskal's condition), see  \cite{sidiropoulos2017tensor} for an overview.

\subsection{B-spline basis functions and clamped B-splines}
In this work, we refer to spline functions as univariate functions consisting of piecewise polynomials, which connect at knots. The collection of all such spline functions of a certain degree $d$ makes up the degree-$d$ spline function space, denoted by $\mathbb{S}_d$. 

A basis spline or B-spline function $B^{\boldsymbol{\Delta}}_{i,d}(u): \mathbb{R} \rightarrow \mathbb{R}$ of degree $d$ is then defined recursively \cite{de1978practical},
\begin{align}
    B^{\boldsymbol{\Delta}}_{i,d}(u) = \dfrac{u-u_i}{u_{i+d} - u_i} B^{\boldsymbol{\Delta}}_{i,d-1}(u) + \dfrac{u_{i+d+1}-u}{u_{i+d+1} - u_{i+1}} B^{\boldsymbol{\Delta}}_{i+1,d-1}(u) \nonumber
\end{align}
with knot vector $\boldsymbol{\Delta} = \begin{bmatrix}
    u_1 & u_2 & \cdots & u_k
\end{bmatrix} \in \mathbb{R}^{k}$ and for which the recursion stops at B-splines of degree $0$ with value
\begin{align}
    B^{\boldsymbol{\Delta}}_{i,0}(u) = \begin{cases}
        1 & \text{if } u_i \leq u < u_{i+1} \\
        0 & \text{otherwise}
    \end{cases}. \nonumber
\end{align}
The functions $B^{\boldsymbol{\Delta}}_{i,d}(.)$ are called basis splines or B-splines since any spline $S_{\boldsymbol{\Delta}}(u) \in \mathbb{S}_d$, defined on a given knot vector $\boldsymbol{\Delta}$, can be represented by a linear combination of $d$th-degree B-splines,
\begin{align}
    S_{\boldsymbol{\Delta}}(u) = \sum^{\nu}_{i=1} c_i B^{\boldsymbol{\Delta}}_{i,d}(u). \nonumber
\end{align}
where $\nu$ is called the \textit{degrees of freedom} of the spline. Thus, the B-splines form a basis of the spline function space \cite{de1978practical}. This can be written more explicitly using the notation $\mathbb{S}_{d,\boldsymbol{\Delta}} \subset \mathbb{S}_{d}$ which restricts the spline function space to those defined on the knot vector $\boldsymbol{\Delta}$, then 
\begin{align}
    \text{span}(\{B^{\boldsymbol{\Delta}}_{1,d}, B^{\boldsymbol{\Delta}}_{2,d}, \cdots, B^{\boldsymbol{\Delta}}_{\nu,d}\}) := \text{span}(S_{d, \boldsymbol{\Delta}}) = \mathbb{S}_{d,\boldsymbol{\Delta}}.\nonumber
\end{align}
Note here that the degrees of freedom value $\nu$ is the same for all spline functions in $\mathbb{S}_{d,\boldsymbol{\Delta}}$. This is because of the following relation between the degree $d$, internal knots, and degrees of freedom $\nu$
\begin{align}
    \text{internal\_knots} = \nu - d - 1, \nonumber
\end{align}
and since for $\mathbb{S}_{d,\boldsymbol{\Delta}}$ the degree $d$ and knot vector $\boldsymbol{\Delta}$ are fixed (thus also the internal\_knots), this leads to a fixed $\nu =  \text{internal\_knots} + d + 1$.

This work focuses on the use of clamped B-splines. This means that for a B-spline function of degree $d$ with knot vector $\boldsymbol{\Delta} = \begin{bmatrix}
    u_1 & u_2 & \cdots & u_k
\end{bmatrix}$, the first and last $d+1$ knots are equal, i.e,
\begin{align}
    u_1=u_2=\cdots=u_{d+1}  \text{ and } u_{k-d+1} = u_{k-d} = \cdots = u_k \nonumber 
\end{align} 
For extrapolation, the boundary polynomial pieces are simply extended. 

\section{Single-layer decoupling and tensor-based solution}\label{sec:single-layer-decoupling}

\subsection{Single-layer decoupling} 

This work focuses on the single-layer decoupling problem, introduced by \cite{dreesen2015decoupling}. The problem is stated as follows: given a multivariate vector function 
\begin{align}
    \mathbf{f}: \mathbb{R}^{m} \rightarrow \mathbb{R}^n: \mathbf{x} \mapsto \begin{bmatrix}
        f_1(\mathbf{x}) & f_2(\mathbf{x}) & \cdots & \mathbf{f}_n(\mathbf{x})
    \end{bmatrix}^{\top}, \nonumber
\end{align}
find linear transformation matrices $\mathbf{W}_1 \in \mathbb{R}^{n \times r}$, $\mathbf{W}_0 \in \mathbb{R}^{r \times m}$ and univariate functions $g_1, g_2, \hdots,g_r$, such that 
\begin{align}
    \mathbf{f}(\mathbf{x}) = \mathbf{W}_1 \mathbf{g}(\mathbf{W}_0 \mathbf{x}),\label{eq:single_layer_decoupling}
\end{align}
where for $\mathbf{u} = \mathbf{W}_0 \mathbf{x} \in \mathbb{R}^{r}$, $\mathbf{g}(\mathbf{u}) = \begin{bmatrix}
        g_1(u_1) & g_2(u_2) & \cdots & g_r(u_r)
    \end{bmatrix}^{\top}. \nonumber
$
The representation in \eqref{eq:single_layer_decoupling} is referred to as a single-layer decoupled representation of $\mathbf{f}(\mathbf{x})$ (Figure \ref{fig:decoupling}).
\subsection{Tensor-based solution} 

To solve the single-layer decoupling problem, Dreesen et al. \cite{dreesen2015decoupling} use the first-order information of the multivariate vector function $\mathbf{f}(\mathbf{x})$. This first-order information is encapsulated in the Jacobian $\mathbf{J}_{\mathbf{f}}(\mathbf{x}) \in \mathbb{R}^{n \times m}$,
\begin{align}
    \mathbf{J}_{\mathbf{f}}(\mathbf{x}) = \begin{bmatrix}
        \dfrac{\partial \mathbf{f}_1(\mathbf{x})}{\partial x_1} & \dfrac{\partial \mathbf{f}_1(\mathbf{x})}{\partial x_2} & \cdots & \dfrac{\partial \mathbf{f}_1(\mathbf{x})}{\partial x_m} \\
        \dfrac{\partial \mathbf{f}_2(\mathbf{x})}{\partial x_1} & \dfrac{\partial \mathbf{f}_2(\mathbf{x})}{\partial x_2} & \cdots & \dfrac{\partial \mathbf{f}_2(\mathbf{x})}{\partial x_m} \\
        \vdots & \vdots &  & \vdots \\
        \dfrac{\partial \mathbf{f}_n(\mathbf{x})}{\partial x_1} & \dfrac{\partial \mathbf{f}_n(\mathbf{x})}{\partial x_2} & \cdots & \dfrac{\partial \mathbf{f}_n(\mathbf{x})}{\partial x_m}
    \end{bmatrix}, \nonumber
\end{align}
which under the assumption that $\mathbf{f}(\mathbf{x})$ admits the decoupled model \eqref{eq:single_layer_decoupling} and evaluated in a point $\mathbf{x}^{(s)}$, has the following structure
\begin{align}
    \mathbf{J}_{\mathbf{f}}(\mathbf{x}^{(s)}) &= \mathbf{W}_1 \; \text{diag}(\mathbf{g}'(\mathbf{W}_0 \mathbf{x}^{(s)})) \; \mathbf{W}_0 \nonumber \\
    &= \mathbf{W}_1 \; \mathbf{D}^{(s)}_{\mathbf{g}} \; \mathbf{W}_0 \label{eq:structure_Jac}
\end{align}
where $\mathbf{g}'(\mathbf{u}) = \begin{bmatrix}
        g'_1(u_1) & g'_2(u_2) & \cdots & g'_r(u_r)
    \end{bmatrix}$. 
Now, evaluating the Jacobian $\mathbf{J}_{\mathbf{f}}(\mathbf{x})$ in $S$ sampling points $\mathbf{x}^{(1)}$, $\mathbf{x}^{(2)}$,$\hdots$, $\mathbf{x}^{(S)}$ yields Jacobians $\mathbf{J}_{\mathbf{f}}(\mathbf{x}^{(1)})$, $\mathbf{J}_{\mathbf{f}}(\mathbf{x}^{(2)})$, $\hdots$, $\mathbf{J}_{\mathbf{f}}(\mathbf{x}^{(S)})$, which can be stacked as frontal slices a tensor $\mathcal{J} \in \mathbb{R}^{n \times m \times S}$. Under the assumption that $\mathbf{f}(\mathbf{x})$ admits the decoupled model \eqref{eq:single_layer_decoupling} and due to the resulting structure of the Jacobian \eqref{eq:structure_Jac}, the frontal slices of the tensor $\mathcal{J}$ correspond to
\begin{align}
   \mathcal{J}_{:,:,s} = \mathbf{W}_1 \; \mathbf{D}^{(s)}_{\mathbf{g}} \; \mathbf{W}_0, \text{ for } s = 1,\hdots, S, \nonumber
\end{align}
which is exactly the structure of the CPD as shown by \eqref{eq:CPD_diag}. As a result and by construction, the tensor $\mathcal{J}$ admits a CPD, i.e., $\mathcal{J}=[\![ \mathbf{W}_1, \mathbf{W}^{\top}_0, \mathbf{G}]\!]$, where $\mathbf{G}^{:,s}=\text{diag}(\mathbf{D}^{(s)}_{\mathbf{g}})$. This shows that by computing the CPD of the tensor $\mathcal{J}$, the linear transformation matrices $\mathbf{W}_1$ and $\mathbf{W}_0$ can be retrieved as well as information about the internal functions captured by the factor matrix $\mathbf{G}$.

This tensor-based solution strategy introduced by \cite{dreesen2015decoupling} can be summarized as follows
\begin{enumerate}
    \item Evaluate the Jacobian $\mathbf{J}_{\mathbf{f}}(\mathbf{x})$ in sampling points $\mathbf{x}^{(1)}$,$\hdots$, $\mathbf{x}^{(S)}$, yielding evaluations $\mathbf{J}_{\mathbf{f}}(\mathbf{x}^{(s)})$,$\hdots$, $\mathbf{J}_{\mathbf{f}}(\mathbf{x}^{(s)})$.
    \item Construct the tensor $\mathcal{J}$ and compute its CPD, yielding $\mathbf{W}_1$, $\mathbf{W}_0$ and the matrix $\mathbf{G}$ containing information about the internal functions $g_1$, $g_2$,$\hdots$, $g_r$.
    \item Use the information in $\mathbf{G}$ to retrieve representations for the internal functions $g_1$, $g_2$,$\hdots$,$g_r$ of the decoupling.
\end{enumerate}

\section{B-spline parameterization of internal functions}\label{sec:B-spline_parameterization}

\subsection{B-spline representation and factor matrix structure}

To represent the internal functions $g_1$, $g_2$, $\hdots$, $g_r$, of the single-layer decoupling, this work focuses on a B-spline representation. More concretely, each internal function $g_i$, for $i \in [r]$, is parameterized as
\begin{equation}
    g_i(u) = \sum^{\nu}_{j=1} c_{i,j} B^{\Delta_i}_{j,d}(u), \label{eq:Bspline_representation}
\end{equation}
with $B^{\Delta_i}_{j,d}(.)$, for $j \in [\nu]$, the B-spline basis functions and $d$ and $\Delta_i$ the degree of the spline and the knot vector for $g_i$ respectively. Using representation \eqref{eq:Bspline_representation}, the derivative of $g_i(.)$ can be represented as
\begin{align}
    \dfrac{\partial g_i(u)}{\partial u} = g'_i(u) = \sum^{\nu}_{j=1} c_{i,j} B'^{\Delta_i}_{j,d}(u), \label{eq:Bspline_representation_derivative}
\end{align}
where $\widetilde{B}^{\Delta_i}_{j,d}(u)$ denotes the derivative of the B-spline basis function $B^{\Delta_i}_{j,d}(u)$ to $u$.

Given the B-spline representation \eqref{eq:Bspline_representation} of the internal functions $g_i$, we can represent the total number of parameters of the decoupled model in equation \eqref{eq:single_layer_decoupling} as a function of its inputs $m$, outputs $n$, rank $r$, and degrees of freedom $\nu$ of the internal functions. The total number of parameters in the B-spline decoupling (BSD) model is then given by 
\begin{align}
    P_{BSD}(r, \nu) &= mr + rn + r\nu. \label{eq:parameters_BSD}
\end{align}

The proposed B-spline representation and its derivative, shown in equations \eqref{eq:Bspline_representation} and \eqref{eq:Bspline_representation_derivative}, incur a specific structure on the factor matrices $\mathbf{G}$ and $\mathbf{R}$ of the Jacobian tensor $\mathcal{J}$ and the zeroth-order information matrix $\mathbf{F}$. Specifically, for $\mathbf{u}^{(s)} = \mathbf{W}_0\mathbf{x}^{(s)} \in \mathbb{R}^r$, it holds that
\begin{align}
    \mathbf{G} &= \begin{bmatrix}
        g'_1(\mathbf{u}^{(1)}_1) & g'_2(\mathbf{u}^{(1)}_2) & \cdots & g'_r(\mathbf{u}^{(1)}_r) \\
        g'_1(\mathbf{u}^{(2)}_1) & g'_2(\mathbf{u}^{(2)}_2) & \cdots & g'_r(\mathbf{u}^{(2)}_r) \\
        \vdots & \vdots & & \vdots \\
        g'_1(\mathbf{u}^{(S)}_1) & g'_2(\mathbf{u}^{(S)}_2) & \cdots & g'_r(\mathbf{u}^{(S)}_r)
    \end{bmatrix},\\
    \mathbf{R} &= \begin{bmatrix}
        g_1(\mathbf{u}^{(1)}_1) & g_2(\mathbf{u}^{(1)}_2) & \cdots & g_r(\mathbf{u}^{(1)}_r) \\
        g_1(\mathbf{u}^{(2)}_1) & g_2(\mathbf{u}^{(2)}_2) & \cdots & g_r(\mathbf{u}^{(2)}_r) \\
        \vdots & \vdots & & \vdots \\
        g_1(\mathbf{u}^{(S)}_1) & g_2(\mathbf{u}^{(S)}_2) & \cdots & g_r(\mathbf{u}^{(S)}_r)
    \end{bmatrix}.
\end{align}
Given representations \eqref{eq:Bspline_representation} and \eqref{eq:Bspline_representation_derivative}, for each column $\mathbf{G}^{:,i}$ and  $\mathbf{R}^{:,i}$, for $i=1,2,\hdots,r$,
\begin{align}
     \mathbf{G}^{:,i}
     &= 
     \begin{bmatrix}
        B'^{\boldsymbol{\Delta}_i}_{1,d}(\mathbf{u}^{(1)}_1) & B'^{\boldsymbol{\Delta}_i}_{2,d}(\mathbf{u}^{(1)}_1) & \hdots & B'^{\boldsymbol{\Delta}_i}_{\nu,d}(\mathbf{u}^{(1)}_1) \\
        B'^{\boldsymbol{\Delta}_i}_{1,d}(\mathbf{u}^{(2)}_1) & B'^{\boldsymbol{\Delta}_i}_{2,d}(\mathbf{u}^{(2)}_1) & \hdots & B'^{\boldsymbol{\Delta}_i}_{\nu,d}(\mathbf{u}^{(2)}_1) \\
        \vdots & \vdots & & \vdots \\
        B'^{\boldsymbol{\Delta}_i}_{1,d}(\mathbf{u}^{(S)}_1) & B'^{\boldsymbol{\Delta}_i}_{2,d}(\mathbf{u}^{(S)}_1) & \hdots & B'^{\boldsymbol{\Delta}_i}_{\nu,d}(\mathbf{u}^{(S)}_1) \\
    \end{bmatrix}
    \cdot
    \begin{bmatrix}
        c_{i,1} \\
        \vdots \\
        c_{i,\nu}
    \end{bmatrix}
    =
    \mathbf{B}'_i \cdot \mathbf{c}_i, \label{eq:constr_G} \\
    \mathbf{R}^{:,i}
     &= 
     \begin{bmatrix}
        B^{\boldsymbol{\Delta}_i}_{1,d}(\mathbf{u}^{(1)}_1) & B^{\boldsymbol{\Delta}_i}_{2,d}(\mathbf{u}^{(1)}_1) & \hdots & B^{\boldsymbol{\Delta}_i}_{\nu,d}(\mathbf{u}^{(1)}_1) \\
        B^{\boldsymbol{\Delta}_i}_{1,d}(\mathbf{u}^{(2)}_1) & B^{\boldsymbol{\Delta}_i}_{2,d}(\mathbf{u}^{(2)}_1) & \hdots & B^{\boldsymbol{\Delta}_i}_{\nu,d}(\mathbf{u}^{(2)}_1) \\
        \vdots & \vdots & & \vdots \\
        B^{\boldsymbol{\Delta}_i}_{1,d}(\mathbf{u}^{(S)}_1) & B^{\boldsymbol{\Delta}_i}_{2,d}(\mathbf{u}^{(S)}_1) & \hdots & B^{\boldsymbol{\Delta}_i}_{\nu,d}(\mathbf{u}^{(S)}_1) \\
    \end{bmatrix}
    \cdot
    \begin{bmatrix}
        c_{i,1} \\
        \vdots \\
        c_{i,\nu}
    \end{bmatrix}
    =
    \mathbf{B}_i \cdot \mathbf{c}_i. \label{eq:constr_R}
\end{align}

\subsection{Link with polynomial and piece-wise linear decoupling}

\noindent

Existing decoupling methods frequently rely on specific bases for the internal
functions \(g_i\). Polynomial decoupling methods 
\cite{dreesen2015decoupling, hollander2017multivariate}, employ the monomial basis
\[
    P_d := \{ x, x^2, \dots, x^d \},
\]
while the work of Zniyed et al. \cite{zniyed2021tensor} uses a piecewise-linear
basis consisting of shifted ReLU functions,
\[
    L_d := \{ \mathrm{ReLU}(x - t_1), \dots, \mathrm{ReLU}(x - t_d) \},
\]
for breakpoints \(t_1 < \dots < t_d\).
Our B-spline formulation generalizes both constructions.

To formalize this relation, we introduce the following notation.  
For any basis \(B\), let
\[
    \mathcal{G}_B := \operatorname{span}(B)
\]
denote the corresponding scalar function space.  
The space of single-layer decoupled functions of rank \(r\) with internal
functions in \(\mathcal{G}_B\) is defined as
\[
    \mathcal{D}^r_{B}(m,n)
    = 
    \bigl\{
        f : \mathbb{R}^m \to \mathbb{R}^n
        \;\big|\;
        f(x) = \mathbf{W}_1 \mathbf{g}(\mathbf{W}_0 x),\;
        g_i \in \mathcal{G}_B
    \bigr\}.
\]

\paragraph{Characterization of the ReLU span as a strict B-spline subspace:}
Let \(\Delta = \{t_0, t_1, \dots, t_d\}\) with \(t_0 < t_1\).  
The shifted ReLU functions generate continuous piecewise-linear functions
with breakpoints at \(t_1,\dots,t_d\), but with the structural constraint
that the leftmost segment has zero slope:
\[
    g(x) = c_0 + \sum_{j=1}^d c_j\,\mathrm{ReLU}(x - t_j)
    \quad\Longrightarrow\quad
    g'(x) = 0 \text{ for } x < t_1.
\]
Therefore
\[
    \mathcal{G}_{L_d}
    =
    S^{(0)}_{1,\Delta}
    := 
    \left\{
        s \in S_{1,\Delta}
        \;\middle|\;
        s'(x) = 0 \;\text{for}\; x < t_1
    \right\},
\]
which is a strict subspace of the full degree–1 spline space \(S_{1,\Delta}\).

\paragraph{Relations between single-layer decoupling spaces:}
We now analyze the relation between single-layer decoupling spaces constructed using
polynomial, ReLU, and B-spline activations.

\begin{lemma}
\label{lem:intersection}
For any \(d_1,d_2 \in \mathbb{N}\),
\[
    \mathcal{D}^r_{L_{d_1}}(m,n)
    \;\cap\;
    \mathcal{D}^r_{P_{d_2}}(m,n)
    \;=\;
    \{ \text{constant functions} \}.
\]
\end{lemma}
\begin{proof}
The polynomial activation space is
\(
    \mathcal{P}_{d_2} = \mathcal{G}_{P_{d_2}},
\)
the set of globally smooth polynomials with max degree $d_2$.  
The ReLU activation space is the restricted spline space
\(S_{1,\Delta}^{(0)}\), whose elements have a constant leftmost piece.
The only polynomials with a constant segment on an interval are constant
functions.  
Thus
\(
    \mathcal{G}_{P_{d_2}} \cap \mathcal{G}_{L_{d_1}}
    = \{\text{constants}\},
\)
and the same holds for the corresponding decoupled model classes.
\end{proof}

\begin{lemma}
\label{lem:relu_subset_spline}
Let \(\Delta = \{t_0,t_1,\dots,t_d\}\) with \(t_0<t_1\).  
Then
\[
    \mathcal{D}^r_{L_d}(m,n)
    \subset
    \mathcal{D}^r_{S_{1,\Delta}}(m,n),
\]
and the inclusion is strict.
\end{lemma}
\begin{proof}
Every \(g_i \in \mathcal{G}_{L_d} = S^{(0)}_{1,\Delta}\) is by definition an
element of the full spline space \(S_{1,\Delta}\), hence the inclusion holds.
It is strict since functions in \(S_{1,\Delta}\) may have a nonzero slope on
\((t_0, t_1)\), which is impossible for functions in \(S^{(0)}_{1,\Delta}\).
\end{proof}

\begin{lemma}
\label{lem:polynomial_bspline_equiv}
Let \(\Delta = \{a,b\}\) consist of two knots without internal breakpoints.
Then the degree–\(d\) spline space equals the polynomial space:
\[
    S_{d,\Delta} = \mathcal{P}_d.
\]
Consequently,
\[
    \mathcal{D}^r_{P_d}(m,n)
    =
    \mathcal{D}^r_{S_{d,\Delta}}(m,n).
\]
\end{lemma}

\begin{proof}
With no internal knots, splines in \(S_{d,\Delta}\) consist of a single
polynomial piece of degree at most \(d\).  
Conversely, any degree–\(d\) polynomial admits a representation in the
B-spline basis on \(\Delta\).  
\end{proof}

Taken together, Lemmas~\ref{lem:intersection}, \ref{lem:relu_subset_spline}, and \ref{lem:polynomial_bspline_equiv}
show that the B-spline parameterization generalizes both
piecewise-linear and polynomial decoupling.  
By selecting the degree and knot configuration appropriately, the proposed
B-spline framework can recover the expressive power of
polynomial models, ReLU-based models, or more flexible spline-based models.

\section{Optimization problem and algorithm}\label{sec:opt_algorithm}

\subsection{Coupled matrix-tensor factorization}

Section \ref{sec:single-layer-decoupling} showed that computing the canonical polyadic decomposition (CPD) of the Jacobian tensor $\mathcal{J}$ is essential for obtaining a single-layer decoupled representation of the function $\mathbf{f}(\mathbf{x})$. However, since $\mathcal{J}$ contains only first-order information about $\mathbf{f}(\mathbf{x})$, information about the intercepts of the activation functions is lost, causing the B-spline decoupling to be biased toward zero. To address this issue, we follow the approach of Zniyed et al. \cite{zniyed2021tensor} and incorporate zeroth-order information of the system $\mathbf{f}(\mathbf{x})$ into the optimization problem. This zeroth-order information is encapsulated in the matrix $\mathbf{F}$, with structure
\begin{align}
    \mathbf{F} &= \begin{bmatrix}
        \mathbf{f}(\mathbf{x}^{(1)}) & \mathbf{f}(\mathbf{x}^{(2)}) & \cdots & \mathbf{f}(\mathbf{x}^{(S)}) 
    \end{bmatrix} \nonumber \\
    &= \mathbf{W}_1 \cdot \begin{bmatrix}
        \mathbf{g}(\mathbf{u}^{(1)}) & \mathbf{g}(\mathbf{u}^{(2)}) & \cdots & \mathbf{g}(\mathbf{u}^{(S)})
    \end{bmatrix} \nonumber \\
    &= \mathbf{W}_1 \cdot \mathbf{R}^{\top}. \label{eq:decomposition_F}
\end{align}
where $\mathbf{u}^{(s)}=\mathbf{W}_0\mathbf{x}^{(s)}$. The decomposition of the zeroth-order information matrix $\mathbf{F}$ in equation \eqref{eq:decomposition_F} shares the factor matrix $\mathbf{W}_1$ with the CPD of $\mathcal{J}$. The work of Zniyed et al. \cite{zniyed2021tensor} exploits this structure to formulate the single-layer decoupling problem as a \textit{coupled matrix-tensor factorization} (CMTF) \cite{liu2021tensors}. Combining the tensor $\mathcal{J}$ and zeroth-order information matrix $\mathbf{F}$ in a CMTF yields the following optimization problem
\begin{align}
    \!\!\displaystyle{\min_{\substack{\mathbf{W}_1, \mathbf{W}_0, \\ \mathbf{G}, \mathbf{R}}}} & \;\; \lVert \mathcal{J} - [\![\mathbf{W}_1, \mathbf{W}^{\top}_0\!, \mathbf{G}]\!] \rVert^2 + \lambda \lVert \mathbf{F} - \mathbf{W}_1 \mathbf{R}^{\top} \rVert^2, \label{eq:unconstrained_optimization}
\end{align}
where the tensor and matrix are coupled through the factor matrix $\mathbf{W}_1$ and the hyperparameter $\lambda$ determines the weight given to the matrix factorization. 

If the CPD of $\mathcal{J}$ is non-unique or the computed CPD has a nonzero error, then there is no guarantee that the solution of the optimization problem \eqref{eq:unconstrained_optimization} has the required structure for the factor matrices $\mathbf{G}$ and $\mathbf{R}$, described by equations \eqref{eq:constr_G} and \eqref{eq:constr_R}. To solve this, we add the required structure on the columns of $\mathbf{G}$ and $\mathbf{R}$ as hard constraints in the optimization problem. This yields the final optimization problem
\begin{align}
    \!\!\displaystyle{\min_{\substack{\mathbf{W}_1, \mathbf{W}_0, \\\{\mathbf{c}_j\}^r_{j=1}}}} & \;\; \lVert \mathcal{J} - [\![\mathbf{W}_1, \mathbf{W}^{\top}_0\!, \mathbf{G}]\!] \rVert^2 + \lambda \lVert \mathbf{F} - \mathbf{W}_1 \mathbf{R}^{\top} \rVert^2 \label{eq:final_optimization}\\\vspace{-2px}
     \operatorname{s.t.} \quad &\mathbf{G}^{:,i} = \mathbf{B}'_i \cdot \mathbf{c}_{i} \;\;\;\; \text{for } i = 1,2,\hdots,r, \nonumber \\
     &\mathbf{R}^{:,i} = \mathbf{B}_i \cdot \mathbf{c}_{i} \;\;\;\; \text{for } i = 1,2,\hdots,r, \nonumber
\end{align}
which corresponds to the one proposed by Zniyed et al. \cite{zniyed2021tensor}, but with different constraints as a result of the B-spline basis proposed in this work.

\subsection{Robust alternating least squares algorithm}

To solve the optimization problem \eqref{eq:final_optimization}, we employ a projection‑based alternating least‑squares algorithm inspired by the work of Zniyed et al. \cite{zniyed2021tensor}. The full procedure, referred to as R‑CMTF‑BSD, is given in Algorithm~\ref{alg:CMTF-BSD}. The procedure closely follows the CMTF‑BSD algorithm proposed in \cite{de2025non}, but several modifications are introduced to enhance robustness. Namely, a more general normalization scheme and a specific use of Tikhonov regularization. As will be demonstrated in the experiments, these robustness improvements are crucial for achieving accurate and stable performance in the context of transformer and neural network compression.

\begin{algorithm}[t]
    \caption{R-CMTF-BSD algorithm}
    \label{alg:CMTF-BSD}
    \begin{algorithmic}[1]
        \Require $\mathcal{J} \in \mathbb{R}^{n \times m \times S}, \mathbf{F} \in \mathbb{R}^{n \times S}, df, d, r,\text{samples}\in \mathbb{R}^{m \times S}$, $\epsilon$
        \State $\mathbf{W}_1, \mathbf{W}_0,  \mathbf{G}, \mathbf{R} \gets$ Random initialization
        \While{stop criteria not met}
            \State $\mathbf{W}_1, \mathbf{W}_0, \mathbf{G}, \mathbf{R} \gets \text{Normalize\_Factors}(\mathbf{W}_1, \mathbf{W}_0, \mathbf{G}, \mathbf{R}, \epsilon)$
            \vspace{8px}
            \State $\mathbf{W}_1 \gets \displaystyle{\argmin_{\mathbf{W}_1}} \lVert \text{unfold}_1(\mathcal{J}) - \mathbf{W}_1 \cdot ( \mathbf{G} \odot \mathbf{W}^{\top}_0)^{\top} \rVert^2 + \lambda \lVert \mathbf{F} - \mathbf{W}_1 \mathbf{R}^{\top} \rVert^2 + \mu_{\mathbf{W}} \lVert \mathbf{W}_1 \rVert^2$
            \vspace{2px}
            \State $\mathbf{W}_0 \gets \displaystyle{\argmin_{\mathbf{W}_0}} \lVert \text{unfold}_2(\mathcal{J}) - \mathbf{W}^{\top}_0 \cdot ( \mathbf{G} \odot \mathbf{W}_1)^{\top} \rVert^2  + \mu_{\mathbf{W}} \lVert \mathbf{W}_0 \rVert^2$
            \vspace{2px}
            \State $\mathbf{G} \gets \displaystyle{\argmin_{\mathbf{G}}} \lVert \text{unfold}_3(\mathcal{J}) - \mathbf{G} \cdot (\mathbf{W}^{\top}_0 \odot \mathbf{W}_1)^{\top} \rVert^2$
            \State $\mathbf{R} \gets \displaystyle{\argmin_{\mathbf{R}}} \lVert \mathbf{F} - \mathbf{W}_1 \cdot \mathbf{R}^{\top} \rVert^2$
            \vspace{8px}
            \State $\text{xSamples} \gets \mathbf{W}_0 \cdot$ samples
            \State $\mathbf{G}, \mathbf{R} \gets \text{Bspline\_projection}\left(\mathbf{G}, \mathbf{R}, df, d, \text{xSamples}\right)$
        \EndWhile
    \Ensure $\mathbf{W}_1, \mathbf{W}_0, \mathbf{G}, \mathbf{R}$
    \end{algorithmic}
\end{algorithm}
\begin{algorithm}[t!]
    \caption{Normalize\_Factors}
    \label{alg:normalize_factors}
    \begin{algorithmic}[1]
        \Require $\mathbf{W}_1, \mathbf{W}_0, \mathbf{G}, \mathbf{R}$, $\epsilon$
        \vspace{5px}
        \For{$i=1,2,\hdots,r$}
            \State $\beta_0 \gets \lVert \mathbf{W}_0^{i,:} \rVert_F$ 
            \State $\beta_1 \gets \lVert \mathbf{W}^{:,i}_1 \rVert_F$
            \vspace{8px}
            \State \algorithmicif $ \;\; \beta_0 < \epsilon$ \algorithmicthen $ \;\; \beta_0 \gets \epsilon$
            \State \algorithmicif $ \;\; \beta_1 < \epsilon$ \algorithmicthen $ \;\; \beta_1 \gets \epsilon$
            \vspace{8px}
            \State $\mathbf{W}_0^{i,:} \gets \mathbf{W}_0^{i,:} / \beta_0$
            \State $\mathbf{W}_1^{:,i} \gets \mathbf{W}_1^{:,i} / \beta_1$
            \vspace{8px}
            \State $\mathbf{G}^{:,i} \gets \beta_0\beta_1\mathbf{G}^{:,i}$
            \State $\mathbf{R}^{:,i} \gets \beta_1\mathbf{R}^{:,i}$
        \EndFor
        \vspace{5px}
        \Ensure $\mathbf{W}_0, \mathbf{W}_1, \mathbf{G}, \mathbf{R}$
    \end{algorithmic}
\end{algorithm}
\begin{algorithm}[t]
    \caption{Bspline\_projection}
    \label{alg:Bspline_projection}
    \begin{algorithmic}[1]
        \Require $\mathbf{G}, \mathbf{R}, df, d, \text{xSamples} \in \mathbb{R}^{r \times S}$
        \vspace{5px}
        \For{$i=1,2,\hdots,r$}
            \State $\mathbf{x}_i \gets $ xSamples$^{i,:}$
            \vspace{8px}
            \State Internal\_knots $ \gets \nu - d - 1$
            \State $\mathbf{q} \gets \text{Quantiles}(\mathbf{x}_i, \text{Internal\_knots})$
            \State $\boldsymbol{\Delta}_i \gets \mathbf{0}_{\nu}$
            \For{$j=1,2,\hdots,\nu$}
                \State $\boldsymbol{\Delta}_{i,j} \gets \displaystyle{\argmin_{x}} \lVert \mathbf{q}_{j} - \mathbf{x}_{i,j} \rVert^2$, with $x \in \{ \mathbf{x}_{i,1}, \cdots, \mathbf{x}_{i,S} \}$
            \EndFor
            \vspace{8px}
            \State $\mathbf{B}_i \gets \text{Design\_matrix}(\mathbf{x}_i, \mathbf{\Delta}_i, df, d)$
            \State $\mathbf{B}'_i \gets \text{Differentiate}(\mathbf{B})$
            \vspace{8px}
            \State $\mathbf{c}_i \gets \displaystyle{\argmin_{\mathbf{c}_i}} \lVert \mathbf{G}^{:,i} - \mathbf{B}'_i \cdot \mathbf{c}_i \rVert^2 + \lambda \lVert \mathbf{R}^{:,i} - \mathbf{B}_i \cdot \mathbf{c}_i \rVert^2 + \mu_{\mathbf{c}} \lVert \mathbf{c}_i \rVert^2$
            \State $\mathbf{G}^{:,i} \gets \mathbf{B}'_i \cdot \mathbf{c}_i,\;\; \mathbf{R}^{:,i} \gets \mathbf{B}_i \cdot \mathbf{c}_i$
        \EndFor
        \vspace{5px}
        \Ensure $\mathbf{G}, \mathbf{R}$
    \end{algorithmic}
\end{algorithm}

In Algorithm~\ref{alg:CMTF-BSD}, the factor matrices are first normalized according to Algorithm~\ref{alg:normalize_factors}. This procedure applies a norm-floor strategy that replaces norms below a threshold $\epsilon$ with $\epsilon$, preventing divisions by $0$. Next, the factor matrices $\mathbf{W}_1$, $\mathbf{W}_0$, $\mathbf{G}$, and $\mathbf{R}$ are updated (in that order) using linear least-squares updates, ignoring non-linear dependencies between the factor matrices. These updates correspond to lines $4$ to $7$ in Algorithm~\ref{alg:CMTF-BSD}. The updates of $\mathbf{W}_1$ and $\mathbf{W}_0$ contain an additional Tikhonov regularization term with parameter $\mu_{\mathbf{W}}$. After updating each of the factor matrices, the columns of $\mathbf{G}$ and $\mathbf{R}$, denoted by $\mathbf{G}^{:,i}$ and $\mathbf{R}^{:,i}$ are projected back into the column space of $\mathbf{B}'_i$ and $\mathbf{B}_i$ respectively, for $i=1,\hdots,r$, as described by the constraints in optimization problem \eqref{eq:final_optimization}. This projection occurs in Algorithm~\ref{alg:CMTF-BSD} on line $9$, inside the Bspline\_projection(.) function, which is given in Algorithm~\ref{alg:Bspline_projection}.

For the Bspline\_projection(.) in Algorithm~\ref{alg:Bspline_projection}, an important question is how to determine the knots for the spline basis. Namely, the distribution of inputs to the activations of the decoupling can change every iteration due to the update of $\mathbf{W}_0$. Ideally, the knots are updated whenever the inputs to the activations change. 

Since knot-selection is executed once per activation per iteration, it is important to keep the additional computation cost for the knot-selection low. Thus, we choose to select knots based on quantiles of the inputs. This quantile knot selection corresponds to lines $3$ to $8$ in Algorithm~\ref{alg:Bspline_projection}. Note that lines $6$ to $8$ are essential since the computed quantiles may not be components of the vector of input elements $\mathbf{x}_i$ and thus each quantile is mapped to the closest input. 

On line $11$, the coefficients $\mathbf{c}_i$ of the $i$th internal function $g_i$ are updated based on the columns $\mathbf{G}^{:,i}$ and $\mathbf{R}^{:,i}$. This update includes an additional Tikhonov regularization term, so instead of regularizing the least squares updates of $\mathbf{G}$ and $\mathbf{R}$ directly, we regularize their projections. The updated coefficients are then used in line $12$ to project $\mathbf{G}^{:,i}$ and $\mathbf{R}^{:,i}$ back into the column space of $\mathbf{B}'_i$ and $\mathbf{B}_i$ respectively.

This section introduced a norm-floor hyperparameter $\epsilon$ and two Tikhonov regularization parameters, $\mu_{\mathbf{c}}$ and $\mu_{\mathbf{W}}$. For the remainder of this work, these parameters are fixed to $\epsilon=1e-6$, $\mu_{\mathbf{c}}=1e-5$, and $\mu_{\mathbf{W}}=1e-4$. The following section discusses transformer models and how the R-CMTF-BSD algorithm can be applied to compress them.

\section{Transformers and compression via decoupling} \label{sec:transformer_compression}

\subsection{Transformer models}\label{sec:transformer_model}

Transformer models were originally introduced by Vaswani et al. \cite{vaswani2017attention} in the context of neural sequence transduction tasks such as machine translation. At a high level, a Transformer processes structured input representations (such as a sequence) using stacked transformer blocks, each of which combines an attention-based token mixing operation with a point-wise fully connected neural network (FCNN), connected through residual connections and normalization layers (see the left block in Figure \ref{fig:FCNN_compression}).

In their original formulation, Vaswani et al. \cite{vaswani2017attention} proposed a full encoder–decoder architecture, where both the encoder and the decoder are composed of such transformer blocks, differing mainly in their attention mechanisms. Following this initial work, numerous Transformer variants have been introduced that adapt this building block to different tasks and data modalities. In particular, encoder‑only and decoder‑only architectures have become common. Some examples are vision (ViT) \cite{dosovitskiy2020image} and shifted window (Swin) transformers \cite{liu2021swin} for vision tasks, BERT \cite{devlin2019bert}, RoBERTa \cite{liu2019roberta}, and GPT \cite{brown2020language} for natural language tasks, and Chronos \cite{ansari2024chronos} for time series. 

The targets of this work are transformers for vision tasks, namely ViT and Swin. The next section discusses how and why we apply the decoupling procedure to compress transformer models.

\begin{figure}[h!]
    \centering
    \includegraphics[width=0.65\linewidth]{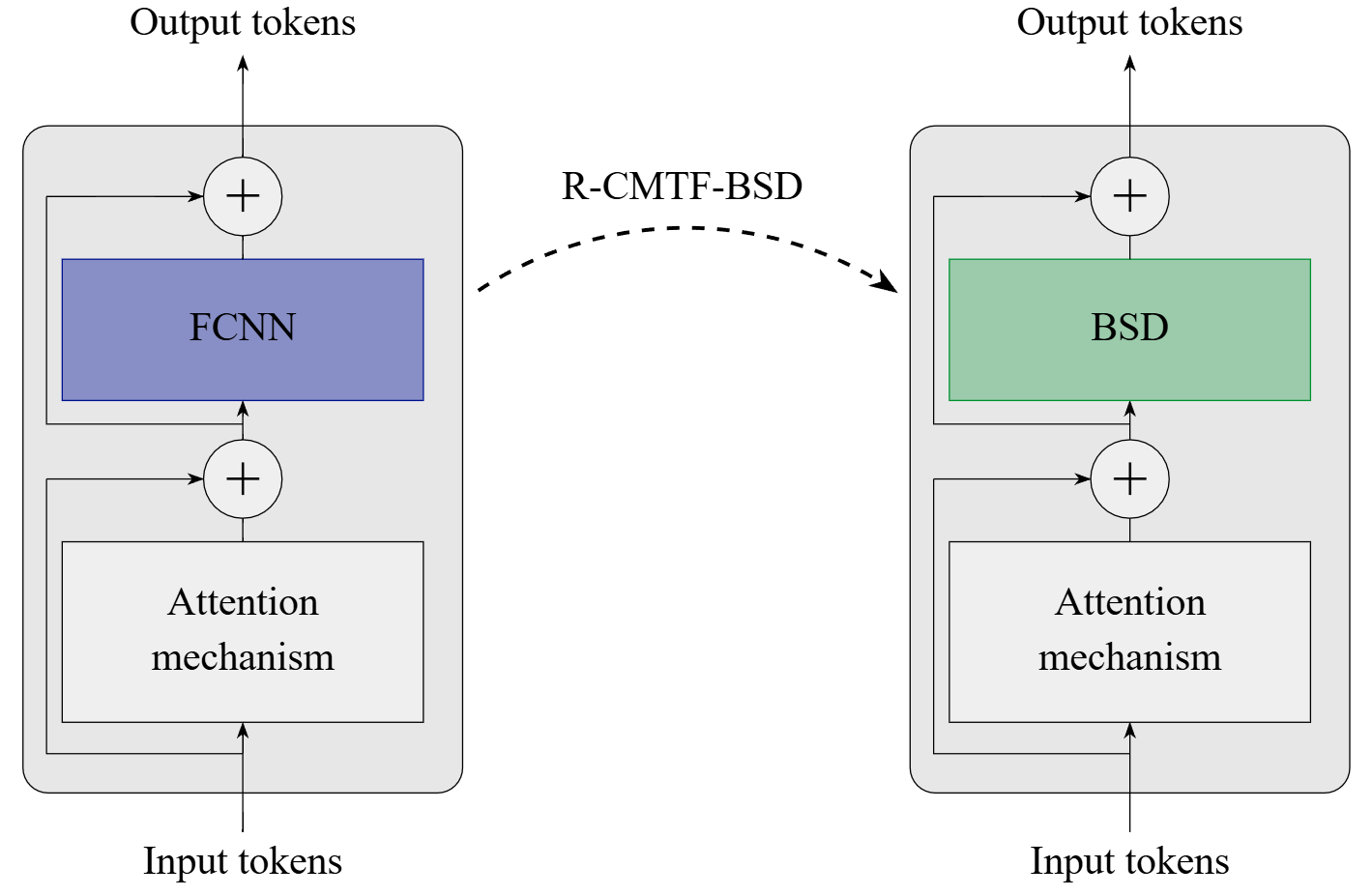}
    \caption{General transformer encoder block with attention mechanism followed by channel-wise FCNN (Left) and compressed transformer block with the FCNN replaced by a B-spline decoupling, computed by the R-CMTF-BSD Algorithm~\ref{alg:CMTF-BSD} (right).}
    \label{fig:FCNN_compression}
\end{figure}

\subsection{B-spline decoupling based compression}\label{sec:B-spline_compression}

We apply the R-CMTF-BSD Algorithm~\ref{alg:CMTF-BSD} to compress and replace the fully connected neural network (FCNN) blocks inside transformer architectures (see Figure \ref{fig:FCNN_compression}). The computed decoupled representation serves as an approximate replacement of the original FCNN.

Each transformer block takes an input $\mathbf{x} \in \mathbb{R}^m$, where $m$ is the embedding dimension, and passes it through an FCNN that expands the input by a factor $\delta \ge 1$ (the MLP ratio). After the hidden layer, the output is projected back to dimension $m$. Thus, each FCNN block can be written as a vector-valued function $\mathbf{f}: \mathbb{R}^{m} \rightarrow \mathbb{R}^{m}$,
\begin{align}
    \mathbf{f}(\mathbf{x}) = \mathbf{V}_1 \mathbf{h}(\mathbf{V}_0\mathbf{x}), \label{eq:fcnn_function}
\end{align}
where $\mathbf{V}_1 \in \mathbb{R}^{m \times \delta m}$, $\mathbf{V}_0 \in \mathbb{R}^{\delta m \times m}$ and 
$
\mathbf{h}(\mathbf{u}) = \begin{bmatrix}
    h(u_1) & h(u_2) & \cdots & h(u_{\delta m})
\end{bmatrix} \in \mathbb{R}^{\delta m}
$
collects fixed univariate activation functions. The number of parameters of the FCNN, denoted by equation \eqref{eq:fcnn_function}, is
\begin{align}
    P_{FC}(\delta, m) = 2\delta m^2 \label{eq:params_FCNN}
\end{align}
For compression using CMTF-BSD, we can consider two cases of interest. First, for fixed rank $r$, we seek conditions on the spline degrees of freedom $\nu$ such that
\begin{align}
    P_{BSD}(r, \nu) < P_{FC}(\delta, m). \label{eq:condition_r}
\end{align}
Similarly, for fixed $\nu$, we may seek conditions on $r$ such that \eqref{eq:condition_r} holds. The following lemma characterizes these conditions.
\begin{lemma}
    Inequality \eqref{eq:condition_r} holds if and only if
    \begin{equation}
        r < \dfrac{\delta m^2}{m + \frac{\nu}{2}} \quad \text{ and } \quad \nu < \dfrac{2m(\delta m - r)}{r}.
    \end{equation}
    \begin{proof}
        Substituting the expressions \eqref{eq:params_FCNN}, \eqref{eq:parameters_BSD} of $P_{FC}(\delta, m)$ and $P_{BSD}(r, \nu)$ into \eqref{eq:condition_r} gives
        $
            2 r  m + r \nu < 2 \delta m^2. \nonumber
        $
        Rearranging gives the stated bounds. 
    \end{proof}
\end{lemma}

\subsection{Full network compression procedure}\label{sec:procedure_compression}

As mentioned in Section \ref{sec:B-spline_compression}, in the compression setting, the decoupled representation serves as an approximate replacement of the original FCNN. As a result, the overall model performance (accuracy, F1 score, etc) may decrease slightly. However, because the decoupled representation remains a feedforward neural network with flexible activation functions, it still allows for backpropagation when integrated into the transformer model. This allows the modified model to be fine‑tuned on the full dataset, recovering part of the lost performance.

The full network compression procedure can be summarized as follows:

\begin{enumerate}
    \item \textbf{Select an FCNN within a transformer-encoder block}, e.g., the one with the largest amount of parameters or greatest compression potential.
    \item \textbf{Choose a small subset of training data}, e.g., a fixed number of samples per class, and compute the Jacobian $\mathcal{J}$ and zeroth-order information matrix $\mathbf{F}$ for the chosen FCNN.
    \item \textbf{Compute a decoupled representation of the FCNN} using the R-CMTF-BSD algorithm with $\mathcal{J}$ and $\mathbf{F}$.
    \item \textbf{Replace the original FCNN in the ViT with the decoupled version} and measure how the substitution affects performance.
    \item \textbf{If needed, finetune the ViT} until performance reaches an acceptable threshold.
    \item \textbf{Repeat steps 1-5} for FCNNs of remaining transformer-encoder blocks if additional compression is desired.
\end{enumerate}

This work evaluates two simple compression procedures, termed \textit{front to back} (FB) and \textit{back to front} (BF). Both share some design choices: for step $2$, the subset of training data is a randomly selected subset of $S$ data points, for step $5$, the full network is finetuned for $5$ epochs after each compression step. The methods differ only in compression order: FB compresses transformer blocks sequentially from the first to the last block, while BF starts from the final block and proceeds toward the front.

\begin{remark}
    While not the focus of this work, more principled approaches to the compression procedure are possible. For example, a method based on the Fisher information metric \cite{amari2000methods} of each block, i.e., in each step, determines and compresses the block with the lowest Fisher information.
\end{remark}
    
\section{Experiments}\label{sec:experiments}

\subsection{General setup and evaluation metrics}

The following sections use CMTF-BSD to refer to the version of the R-CMTF-BSD Algorithm~\ref{alg:CMTF-BSD} without normalization and without Tikhonov regularization. These algorithms are used to compute decoupled representations $\widehat{\mathbf{f}}(\mathbf{x})$ that approximate an FCNN, represented by the function $\mathbf{f}(\mathbf{x})$, of the transformer block(s) under consideration. The results of the B-spline decoupling are compared with the polynomial decoupling results. The polynomial decoupling is computed using the structure of the R-CMTF-BSD Algorithm~\ref{alg:CMTF-BSD}, but with polynomial internal functions of a predetermined degree $d$. The polynomial setup is referred to as R-CMTF-PD in the following sections, or as CMTF‑PD when normalization and regularization are omitted.

The computed decoupling is evaluated based on normalized mean-squared errors of $\mathcal{J}$ and $\mathbf{F}$, denoted by Error($\mathcal{J}$) and Error($\mathbf{F}$) respectively, and defined as
\begin{equation}
    \text{Error}(\mathcal{J}) = \dfrac{\lVert \mathcal{J} - \hat{\mathcal{J}} \rVert^2}{\lVert \mathcal{J} \rVert^2}, \; \text{Error}(\mathbf{F}) = \dfrac{\lVert \mathbf{F} - \hat{\mathbf{F}} \rVert^2}{\lVert \mathbf{F} \rVert^2}. \nonumber
\end{equation}
The approximation quality of the decoupled model with respect to the function $\mathbf{f}(\mathbf{x})$ is evaluated using the relative root mean-squared error $e_i$ per output $f_i(\mathbf{x})$, defined as a percentage
\begin{equation}
    e_i = \sqrt{\dfrac{\sum^S_{s=1}(f_i(\mathbf{x}^{(s)}) - \widehat{f}_i(\mathbf{x}^{(s)}))^2}{\sum^S_{s=1}(f_i(\mathbf{x}^{(s)}) - \mathbb{E}[f_i])^2}} \times 100, \nonumber
\end{equation}
with $\widehat{\mathbf{f}}(\mathbf{x})$ the computed decoupling that approximates $\mathbf{f}(\mathbf{x})$. To evaluate network compression, we introduce two metrics, the savings percentage (SP) and accuracy drop (AD), defined as 
\begin{gather}
    \text{SP} = \left(1 - \dfrac{\text{\# Parameters in decoupling}}{\text{\# Parameters in original model}}\right) \times 100, \nonumber \\
    \text{AD} = Acc_B - Acc_A, \nonumber
\end{gather}
where $Acc_B$ and $Acc_A$ denote the accuracy of the model under consideration before and after compression, respectively.

\subsection{Synthetic example}

We start with a simple synthetic example to show the B-spline decoupling procedure in a function approximation context. For the example, the R-CMTF-BSD (Algorithm~\ref{alg:CMTF-BSD}) and R-CMTF-PD algorithm compute a decoupled representation of the following system $\mathbf{f}: \mathbb{R}^2 \rightarrow \mathbb{R}^2$,
\begin{align}
    \mathbf{f}(\mathbf{x}) = \begin{bmatrix}
        f_1(\mathbf{x}) \\
        f_2(\mathbf{x})
    \end{bmatrix} = \begin{bmatrix}
        sin\left(0.2 \cdot \pi \cdot (x_1 + x_2) + 2.5\right) \\
        tanh(x_1 + x_2) - 5
    \end{bmatrix}. \label{eq:synthetic}
\end{align}

Figure \ref{fig:synthetic_example} presents the component functions $\widehat{f}_1$ and $\widehat{f}_2$, obtained with the R-CMTF‑PD algorithm using a polynomial of degree $d=10$ and with the R-CMTF‑BSD algorithm using B‑splines with degrees of freedom $\nu = 11$ and spline degrees $d=10$, $1$, and $3$. These choices of B-spline correspond to, respectively, a single polynomial, a piecewise linear representation, and a cubic B‑spline representation, highlighting their versatility in configuration. Both the polynomial and the B‑spline, $d=10$, accurately approximate $f_1$, although the polynomial exhibits noticeable deviations in the tails of $f_2$. As expected, the piecewise linear representation captures the general shape of the functions but lacks precision. The cubic B‑spline provides a good approximation of $f_1$ and avoids the tail deviations observed for $f_2$ in the polynomial case.

Table \ref{tab:synthetic_example} reports the relative root‑mean‑squared errors $e_i$ for each output, confirming these observations. In the polynomial and B‑spline, $d=10$ cases, the error for the second output exceeds $13\%$, with both methods yielding similar performance.  The piecewise linear B‑spline results in larger errors for both outputs, reflecting its limited suitability for regression problems. In contrast, the cubic B‑spline representation achieves low errors across both outputs, resulting in the lowest mean error overall.

\begin{figure}
    \centering
    \includegraphics[width=\linewidth]{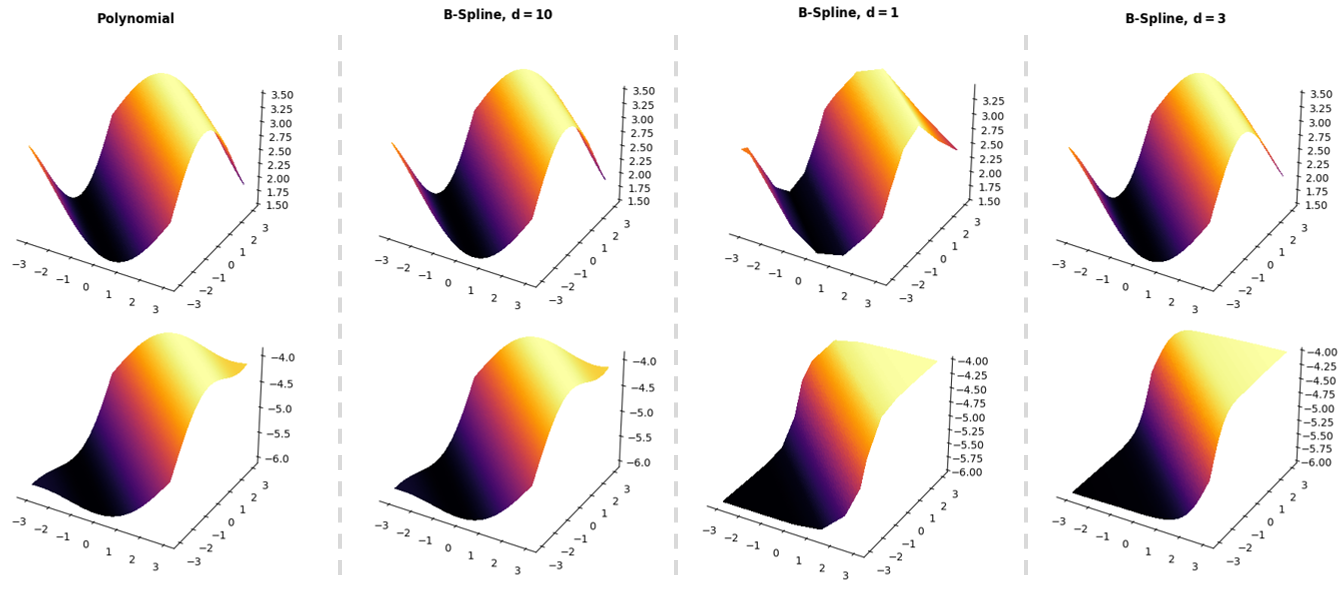}
    \caption{Approximations of the component functions $f_1$ (first row) and $f_2$ (second row) of the system $\mathbf{f}(\mathbf{x})$ in equation \eqref{eq:synthetic}, computed by the R-CMTF-PD algorithm (first column) with $d=10$ and the R-CMTF-BSD algorithm with $\nu = 11$ and $d = 10, 1, 3$ (second, third and fourth column respectively). All decoupled representations use rank $r=3$. The cubic spline representation (last column) provides the most accurate approximation of the underlying functions.}
    \label{fig:synthetic_example}
\end{figure}

\begin{table}[h]
    \centering
    \caption{Relative root mean-squared errors $e_i$ of the computed decouplings $\widehat{\mathbf{f}}(\mathbf{x})$ for the outputs of $\mathbf{f}(\mathbf{x})$ in equation \eqref{eq:synthetic}. For each output, the lowest achieved error is highlighted in bold.}
    \label{tab:synthetic_example}
    \begin{tabular}{|m{1.5cm}|C{2.5cm}|C{1.5cm}|C{1.5cm}|C{1.5cm}|}
        \hline
         & R-CMTF-PD & \multicolumn{3}{c|}{R-CMTF-BSD, $\nu = 11$} \\
         \hhline{-----}
        $\mathbf{f}(\mathbf{x})$ & $d=10$ & $d=10$ & $d=1$ & $d=3$ \\
        \hline
        $e_1$ ($\%$) & $0.39$ & $\mathbf{0.25}$ & $9.78$ & $2.63$ \\
        \hline
        $e_2$ ($\%$) & $13.68$ & $13.64$ & $4.55$ & $\mathbf{1.66}$\\
        \hline
        Mean ($\%$) & $7.04$ & $6.95$ & $7.17$ & $\mathbf{2.15}$\\
        \hline
    \end{tabular}
\end{table}

\subsection{Compressing a single transformer block} \label{subsection:single-transformer-block}

The experiments in this section aim to help understand the behaviour of the B-spline decoupling compression procedure under different hyperparameter configurations. The focus is to perform an ablation study by applying the decoupling procedure to a single transformer block and compress the FCNN module. The experiments use a Vision transformer (ViT) model trained on the MNIST dataset \cite{deng2012mnist}, whose FCNN of the first block is compressed using the (R-)CMTF-BSD algorithm, as well as their polynomial variants (R-)CMTF-PD. This allows to compare results before and after the robustness measures are added, showcasing their necessity. The trained ViT has an accuracy of $97.57\%$ on the test set. Details about model architecture and training procedure are summarized in \ref{app:ViT_details}. 

\begin{figure}
    \centering
    \begin{subfigure}{\textwidth}
        \includegraphics[width=\linewidth]{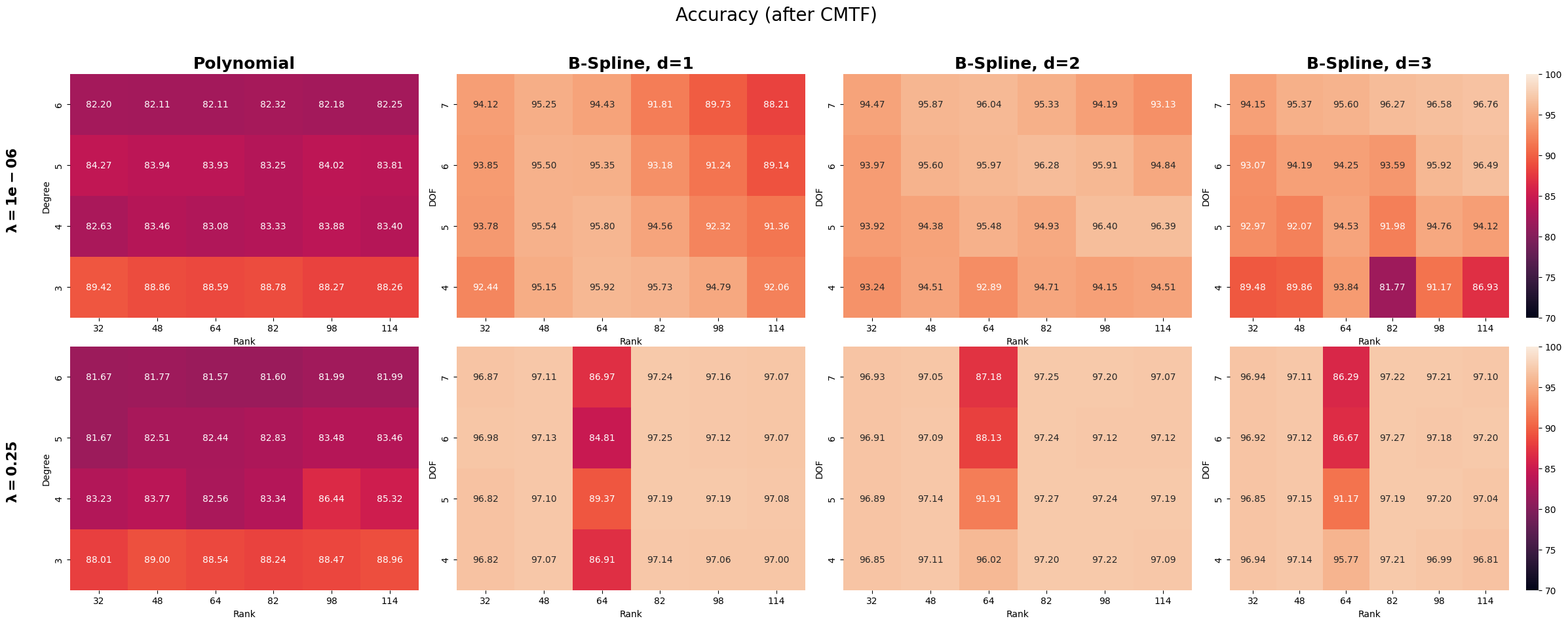}
    \end{subfigure}
    \vspace{5px}
    \textcolor{gray}{\hdashrule[0.5ex]{\linewidth}{1pt}{4pt 4pt}}
    \vspace{5px}
    \begin{subfigure}{\textwidth}
        \includegraphics[width=\linewidth]{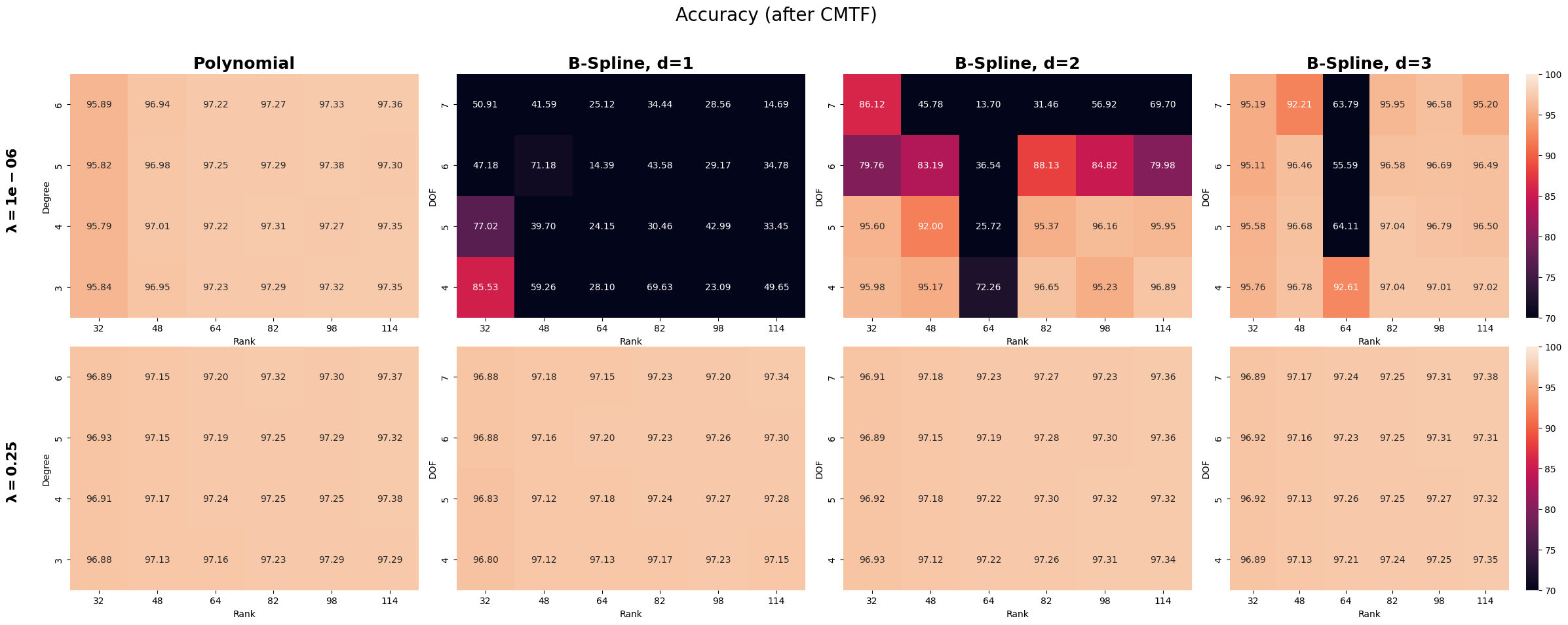}
    \end{subfigure}
    \caption{Mean top-$1$ accuracy results for the (top) CMTF-BSD and CMTF-PD algorithms and (bottom) R-CMTF-BSD and R-CMTF-PD algorithms. Each result is over $5$ runs for different hyperparameter configurations. Results in row one have $\lambda$ set to $1e-6$, in row two $\lambda$ is set to $0.25$. The first column contains (R-)CMTF-PD results with polynomials of degree $d\in \{3, 4, 5, 6\}$ (y-axis), columns two, three and four contain (R-)CMTF-BSD results with degrees of freedom $\nu \in \{4, 5, 6, 7\}$ (y-axis) and degrees $1$, $2$ and $3$ for the B-spline activations respectively. For each heatmap, the x-axis denotes the considered rank values $r\in \{20, 40, 64, 80, 100, 120\}$. The results show that CMTF-BSD is more accurate and stable than CMTF-PD, which suffers from numerical instability. The robust variants remove these issues, with R‑CMTF-PD and R-CMTF-BSD performing best and comparably at $\lambda=0.25$.}
    \label{fig:ablation_acc}
\end{figure}

For the ablation study, the decoupling algorithms are executed under different hyperparameter configurations. More specifically, all possible combinations of $\lambda \in \{1e-6, 0.25\}$ and rank $r \in \{32, 48, 64, 82, 98, 114\}$ are used for the (R-)CMTF-BSD and (R-)CMTF-PD algorithm. For each execution of the algorithms, all internal functions use the same number of coefficients. Namely, each internal function is represented with $4$ to $7$ coefficients, i.e., the degree $d$ of the polynomial activations is varied from $3$ to $6$ and the degrees of freedom $\nu$ of the B-spline representation is varied from $4$ to $7$. The degree of the B-spline representation is also varied (which doesn't change the number of coefficients) from $1$ to $3$, yielding piece-wise linear, quadratic and cubic splines. 

For a more direct comparison of results, the same Jacobian tensor $\mathcal{J}$ and zeroth-order information matrix $\mathbf{F}$ are used for each run, with the main difference between runs being the values with which the decoupling factor matrices are initialized. To construct $\mathcal{J}$ and $\mathbf{F}$, $S=128$ sampling points are used, chosen randomly from the training set. Thus, $\mathcal{J} \in \mathbb{R}^{64 \times 64 \times 128}$ and $\mathbf{F} \in \mathbb{R}^{64 \times 128}$ and the amount of samples per MNIST class can differ.

The heatmaps in Figure \ref{fig:ablation_acc} display the mean accuracy results, after compressing the first transformer block, for the discussed hyperparameter configurations over $5$ executions of the respective decoupling for both the standard (top) and robust versions (bottom) of the algorithms. Similarly, Figures \ref{fig:ablation_Jac} and \ref{fig:ablation_F} show the mean Error($\mathcal{J}$) and Error($\mathbf{F}$) for only the R-CMTF-BSD and R-CMTF-PD algorithms over the same $5$ executions as in Figure \ref{fig:ablation_acc}. 

Figure \ref{fig:ablation_acc} shows that for the non‑robust polynomial case, the results are consistently poor, regardless of the hyperparameter settings. The non‑robust B‑spline approach shows some improvement over the polynomial formulation, particularly for $\lambda=0.25$. However, a numerical issue arises when the rank $r=64$, which is exactly when inputs $=$ outputs $=$ rank. Also, for $\lambda=1e-6$, the B-spline results are not as stable as we would expect, likely due to numerical instabilities.

\begin{figure}[h]
    \centering
    \includegraphics[width=\linewidth]{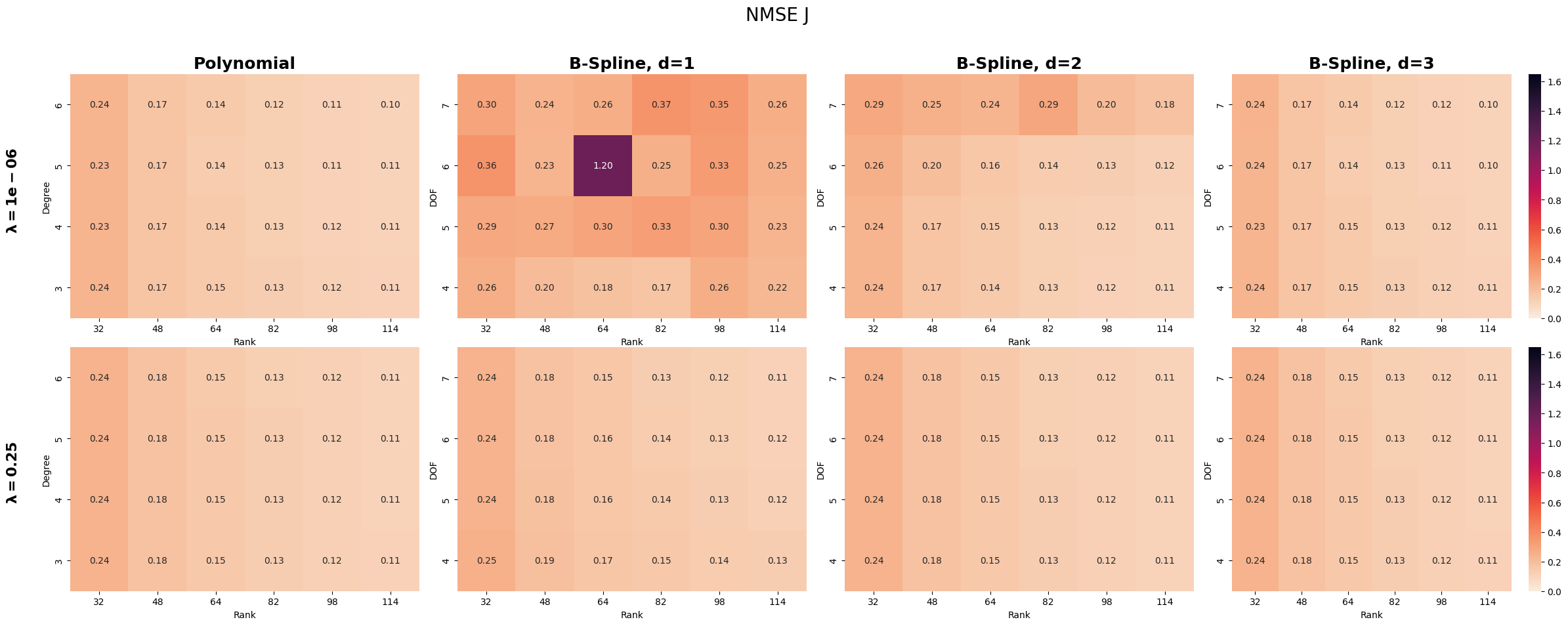}
    \caption{Mean Error($\mathcal{J}$) results for the R-CMTF-BSD and R-CMTF-PD algorithms over $5$ runs for different hyperparameter configurations. The figure layout and axis are as mentioned in the caption of Figure \ref{fig:ablation_acc}. The results indicate that the R‑CMTF-PD and R-CMTF-BSD algorithms perform best and comparably at $\lambda=0.25$.}
    \label{fig:ablation_Jac}
\end{figure}
\begin{figure}[h]
    \centering
    \includegraphics[width=\linewidth]{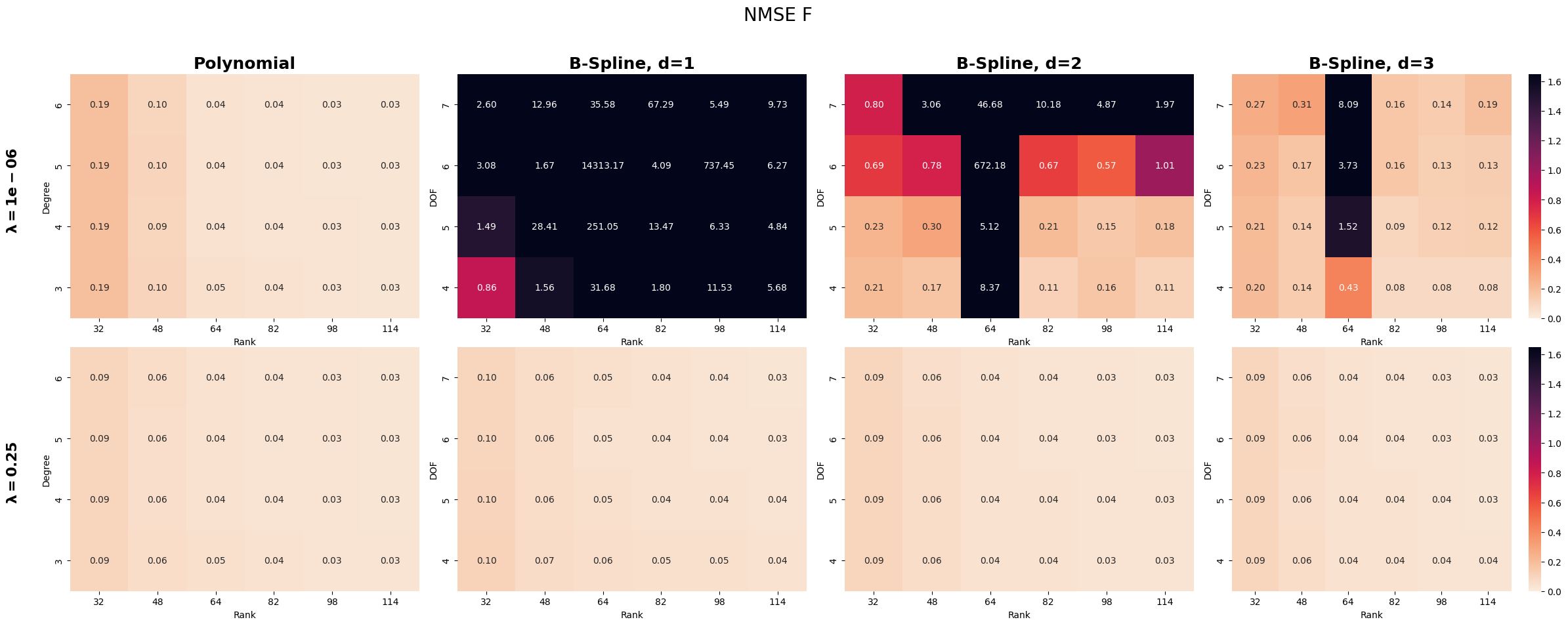}
    \caption{Mean Error($\mathbf{F}$) results for the R-CMTF-BSD and R-CMTF-PD algorithms over $5$ runs for different hyperparameter configurations. The figure layout and axis are as mentioned in the caption of Figure \ref{fig:ablation_acc}. Similarly to Figures \ref{fig:ablation_acc} and \ref{fig:ablation_Jac}. The results indicate that the R‑CMTF-PD and R-CMTF-BSD algorithms perform best and comparably at $\lambda=0.25$.}
    \label{fig:ablation_F}
\end{figure}

The robust results (bottom) in Figure \ref{fig:ablation_acc} indicate different behaviour, particularly for the polynomial case. Here, the normalization and Tikhonov regularization remove numerical issues, leading to stable performance across all settings. Similar behaviour is observed for R-CMTF-BSD with $\lambda=0.25$. For $\lambda=1-e6$, however, the B-spline results collapse. Although this may appear concerning, it is explained by the fixed Tikhonov parameters $\mu_c = 1e-5$ and $\mu_{\mathbf{W}}=1e-4$, which for low $\lambda$ values are too strong and overpower the fit of the zeroth-order information matrix $\mathbf{F}$. This is reflected by the Error($\mathcal{J}$) and Error($\mathbf{F}$) results in Figures \ref{fig:ablation_Jac} and \ref{fig:ablation_F}: $\mathcal{J}$ is fitted well for both $\lambda$ values but $\mathbf{F}$ is not for $\lambda=1e-6$.

Finally, Figure \ref{fig:compression_ablation} reports the savings percentage (SP) across the different polynomial and B‑spline configurations of the internal functions. The SP is computed relative to the compressed FCNN rather than the full transformer model. The polynomial and B‑spline configurations were chosen to use the same number of parameters, as reflected by identical SP values in both heatmaps. The results show that variations in SP are primarily driven by changes in rank, while changes in degree or degrees of freedom lead only to minor differences. As expected, the impact of varying the degree or degrees of freedom becomes more pronounced at higher ranks.

In summary, the accuracy results in Figure \ref{fig:ablation_acc} highlight the need for the proposed robustness measures to mitigate numerical pathologies in both B-spline and polynomial decoupling. Moreover, to prevent fixed Tikhonov regularization from overpowering the fit of $\mathbf{F}$ and degrading approximation quality, $\lambda$ must be sufficiently large. This behaviour is reflected by the Error($\mathcal{J}$) and Error($\mathbf{F}$) results in Figures \ref{fig:ablation_Jac} and \ref{fig:ablation_F}. Additional results for an intermediate $\lambda=0.01$ value are provided in \ref{app:ablation_extra}.

\begin{figure}
    \centering
    \includegraphics[width=0.75\linewidth]{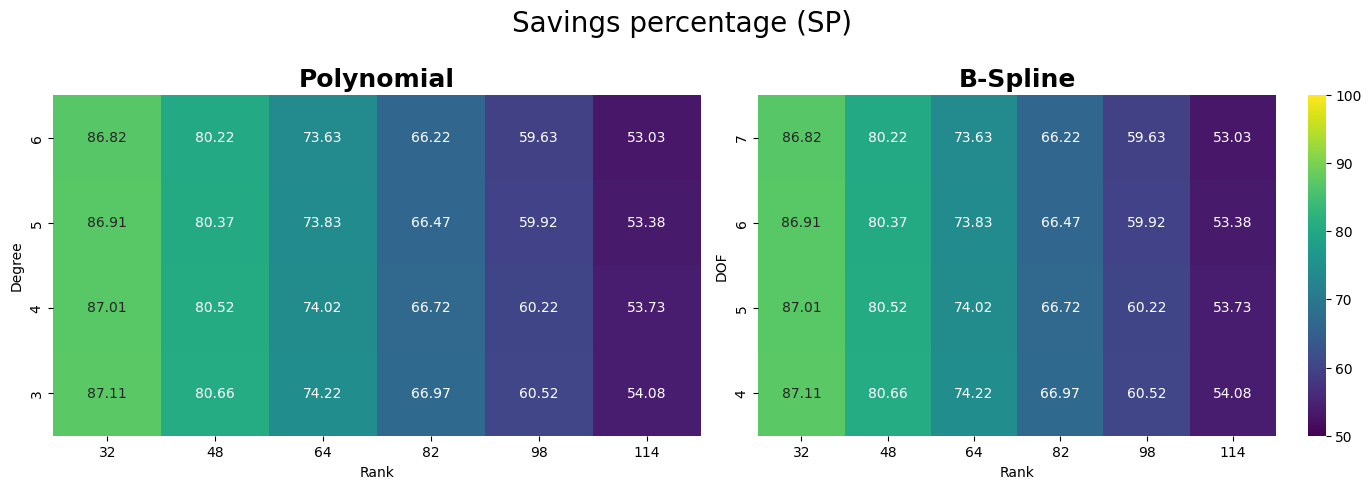}
    \caption{Saving percentage values for the compressed FCNN and used polynomial (left) and B-spline (right) decoupling configurations. As expected, variations in rank drive the largest changes in SP, whereas changes in degree or degrees of freedom have only a marginal effect.}
    \label{fig:compression_ablation}
\end{figure}

\subsection{Full compression procedure}

While the experiments in Subsection~\ref{subsection:single-transformer-block} focused on compressing a single transformer block, the experiments in this section compress a full transformer model using the FB and BF procedures outlined in Subsection~\ref{sec:procedure_compression}.

The experiments consist of two transformer types, a ViT \cite{dosovitskiy2020image} and Swin \cite{liu2021swin} model, trained on two different datasets, SVHN \cite{netzer2011reading} and CIFAR10 \cite{krizhevsky2009learning}. By applying the compression procedures to ViT and Swin, we assess compression behaviour and results across different transformer architectures. Specific data attributes, such as patch size, can be found in \ref{app:data_attributes}. Details on model architecture and training hyperparameters can be found in \ref{app:ViT_details} and \ref{app:Swin_details}

For the experiments, the R-CMTF-PD algorithm is used with polynomials of degree $d=3$, and the R-CMTF-BSD algorithm is used with degrees of freedom $\nu=4$ and piece-wise linear, quadratic, and cubic B-splines, i.e., $d=1,2,3$. For each model and dataset combination, the FB and BF procedures are executed $3$ times, using the same $S=128$ sampling points for each execution. The rank $r$ of the decoupling is chosen such that for the FCNN under consideration, the number of neurons (rank) is equal to the number of inputs and outputs, e.g., for the ViT model trained on SVHN, we have $r=64$.

\begin{figure}
    \centering
    \includegraphics[width=\linewidth]{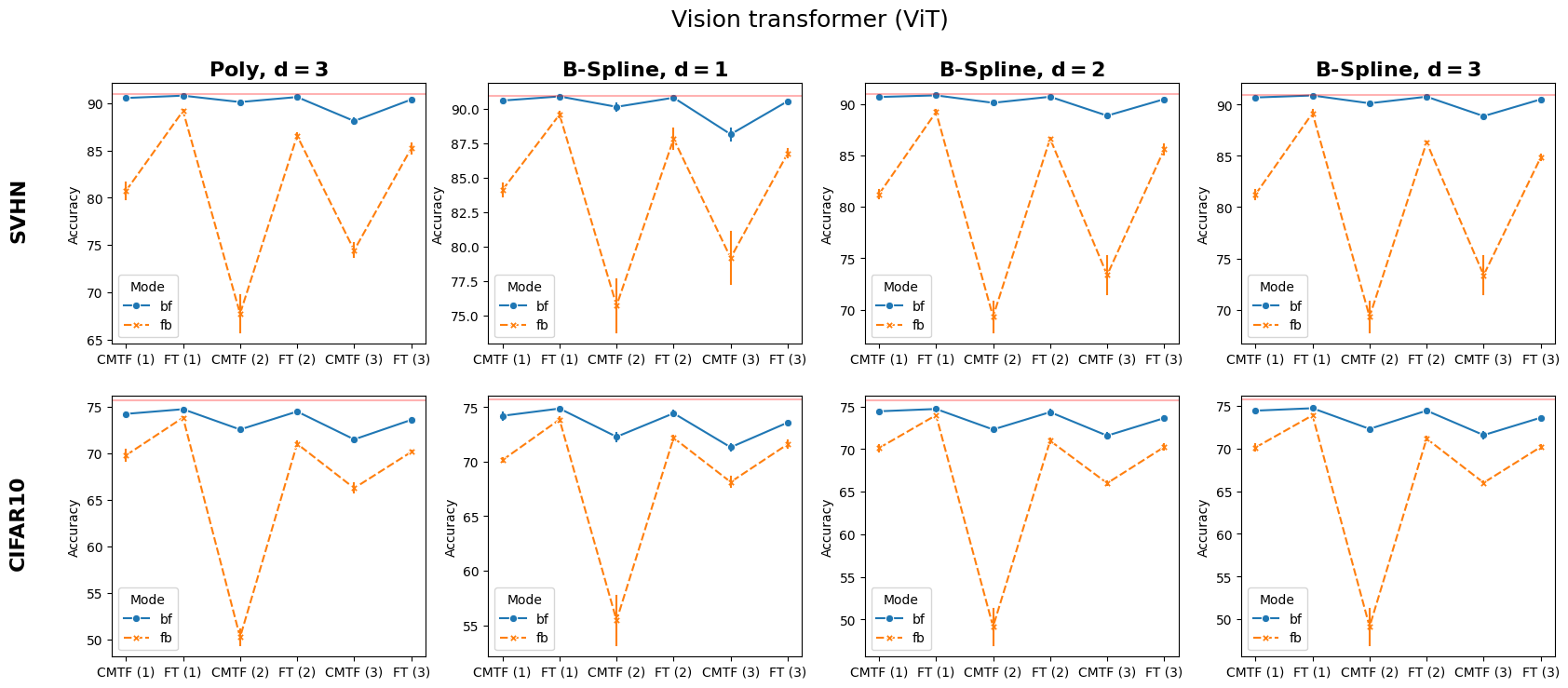}
    \caption{Mean top-$1$ accuracy results over $3$ runs of the BF and FB compression procedures, discussed in Subsection \ref{sec:procedure_compression}, applied to a ViT model using the R-CMTF-PD and R-CMTF-BSD algorithm. Each row corresponds to results on a different dataset, denoted by the row title. The first column corresponds to polynomial activations of degree $3$, the remaining columns correspond to the used B-spline activations of degree $1$, $2$, and $3$ with each spline having $\nu = 4$ degrees of freedom. For each plot, the y-axis denotes the accuracy, and the x-axis indicates the successive steps of the compression procedure: CMTF ($\alpha$) denotes a decoupling step and FT ($\alpha$) denotes a finetuning step, where $\alpha$ indicates the iteration of the compression procedure loop. The results show that for the ViT model, the BF compression procedure achieves more stable and accurate results compared to the FB procedure.}
    \label{fig:Compression_ViT}
\end{figure}

\begin{figure}
    \centering
    \includegraphics[width=\linewidth]{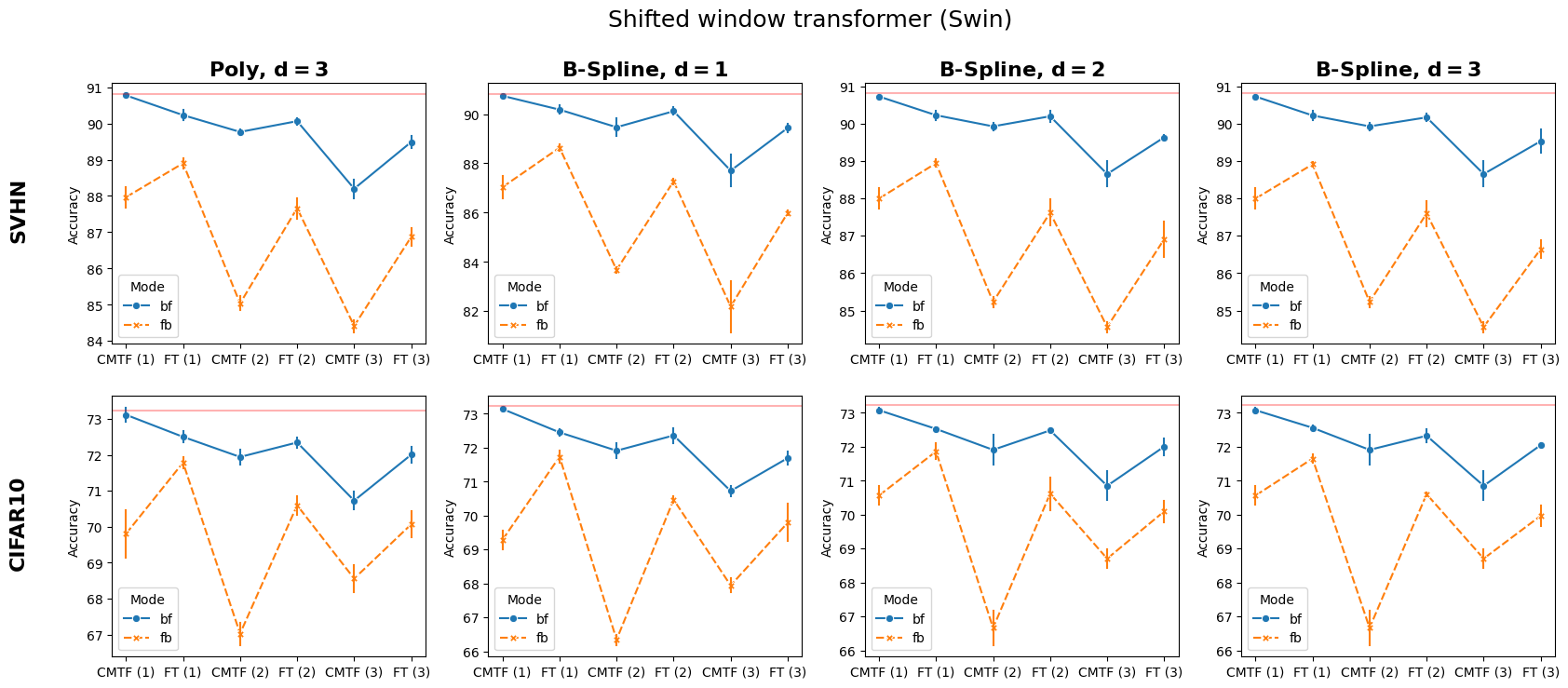}
    \caption{Mean top-$1$ accuracy results over $3$ runs of the BF and FB compression procedures, discussed in Subsection \ref{sec:procedure_compression}, applied to a Swin model using the R-CMTF-PD and R-CMTF-BSD algorithm. The figure layout and axis are as described in the caption of Figure \ref{fig:Compression_ViT}. The results show that for the Swin model, the BF compression procedure achieves more stable and accurate results compared to the FB procedure.}
    \label{fig:compression_Swin}
\end{figure}

Figure \ref{fig:Compression_ViT} shows the experiment results for ViT, Figure \ref{fig:compression_Swin} for Swin. The results in Figures \ref{fig:Compression_ViT} and \ref{fig:compression_Swin} show that, in general, the BF procedure has better accuracy and more stable results compared to the FB procedure. For the FB procedure, the accuracy on the test set after decoupling drops much more compared to the BF procedure, leading to saw-tooth looking compression cycles with the finetuning steps increasing the resulting accuracies. The degree of the spline does not seem to have much effect, as the results of both the BF and FB procedures are quite similar over the B-spline degrees, besides minor differences in standard deviation. The R-CMTF-PD algorithm results are comparable to the R-CMTF-BSD results.

For completeness, \ref{app:full_extra} contains tables with the exact values of the mean accuracies and standard deviations after each BF and FB procedure step, over the $3$ runs, for each model and dataset combination.

\subsection{Comparison with related works}

We compare the proposed B-spline decoupling approach with two related transformer compression methods, namely low-rank decomposition (SVD) \cite{hajimolahoseini2022strategies} and DRONE \cite{chen2021drone}. For the experiments, a data-efficient image transformer (DeiT) \cite{touvron2021training} is trained on the USPS dataset \cite{hull1994database} and used as the baseline model. The comparison results are summarized in Table \ref{tab:deit-usps}.
\begin{table}[htbp!]
\centering
\caption{Comparison of compression approaches on data-efficient image transformers (DeiT) \cite{touvron2021training} on the USPS dataset \cite{hull1994database}
}
\label{tab:deit-usps}
\begin{tabular}{|c|c|c|c|}
\hline
Method                                                      & Top-1 Accuracy    & Number of Parameters $\downarrow$               & Reduction $\uparrow$            \\
\hline
DeiT \cite{touvron2021training}                             & 94.77             & 668650              & 0 \%            \\
\hline
SVD \cite{hajimolahoseini2022strategies}                    & 94.68             & 537322              & 19.64 \%            \\
B-spline                                                    & 94.75             & 524650              & 21.54 \%    \\
\hline
DRONE \cite{chen2021drone}                                  & 94.39             & 318442              & 52.38 \%            \\
B-spline                                                    & 94.41             & 300010              & 55.13 \%    \\
\hline
\end{tabular}%
\end{table}

The SVD-based approach of Hajimolahoseini et al. \cite{hajimolahoseini2022strategies} compresses transformer models by decomposing the fully connected layers into two smaller matrices using singular value decomposition. In contrast to single-shot compression, their work proposes progressive and layer-wise decomposition strategies to better preserve accuracy during fine-tuning. DRONE \cite{chen2021drone} instead performs data-aware low-rank compression by exploiting the observation that intermediate representations of transformer layers often lie in low-dimensional subspaces. Unlike conventional low-rank approximations, DRONE incorporates information from the data distribution when computing the decomposition and can be applied to both fully connected and self-attention layers.

The results in Table \ref{tab:deit-usps} show that the proposed B-spline decoupling achieves competitive or improved accuracy-compression trade-offs compared to both related methods. For moderate compression rates, the B-spline approach achieves a parameter reduction of $21.54\%$ while maintaining a top-$1$ accuracy of $94.75\%$, which is slightly higher than the SVD-based approach despite using fewer parameters. At higher compression rates, the proposed method achieves a reduction of $55.13\%$ with a top-$1$ accuracy of $94.41\%$, again outperforming DRONE both in parameter reduction and final accuracy.

These results suggest that the proposed decoupling strategy provides an effective alternative to purely linear low-rank compression methods. While SVD and DRONE approximate transformer layers through linear subspace decompositions, the proposed method replaces the FCNN blocks with a nonlinear decoupled representation composed of learned univariate B-spline functions. This additional flexibility allows the compressed representation to better preserve the nonlinear behaviour of the original transformer blocks while still achieving substantial parameter reduction.

\section{Conclusions and future work}\label{sec:conclusion}

This work presented B‑spline decoupling as a unifying framework for several existing basis decoupling approaches. In particular, polynomial and piece‑wise linear decoupling were shown to be special cases of the B‑spline formulation. By appropriately configuring the B‑spline basis, the proposed model can reproduce these existing approaches or yield more expressive decoupled representations.

Building on this framework, we introduced R‑CMTF‑BSD, a robust B‑spline decoupling algorithm, and applied it to the compression of transformer models. A case study on a transformer trained on MNIST demonstrated that robust decoupling is essential in this context. The proposed normalization and Tikhonov regularization effectively eliminated numerical instabilities for both polynomial and B‑spline decoupling, leading to improved and more stable performance.

For full transformer compression, we investigated front‑to‑back (FB) and back‑to‑front (BF) strategies. Experiments on ViT and Swin Transformer models trained on SVHN and CIFAR‑10 consistently show the BF approach as superior, yielding higher post‑decoupling accuracy and a more stable accuracy decay across successive decouple–finetune cycles. These results indicate that BF compression combined with decoupling is both effective and promising for transformer architectures.

For future work, the decoupling compression procedure can be applied to a multitude of neural network architectures, e.g., fully connected parts of recurrent neural networks \cite{yu2019review}, variational autoencoders \cite{berahmand2024autoencoders}, MLP-mixers \cite{tolstikhin2021mlp}, or alternative transformer architectures such as BERT \cite{devlin2019bert}. In addition, the decoupling compression procedure can be combined with pruning techniques \cite{pham2025singular,pham2025enhanced} for further compression.

\section*{Acknowledgements}
This work was supported by the Fonds Wetenschappelijk Onderzoek (FWO) fundamental research fellowship 11A2H25N.
\bibliographystyle{elsarticle-harv}
\bibliography{ref}

\appendix
\section{Data and model (hyper)parameters}


\subsection{Attributes of datasets:}\label{app:data_attributes}
\begin{table}[h!]
    \centering
    \caption{Summary of data attributes for the MNIST, SVHN, CIFAR10, and USPS datasets.}
    \begin{tabular}{|c|c|c|c|c|}
        \hline
        \multirow{2}{*}{\textbf{ViT}} & \multicolumn{4}{c|}{Dataset} \\
         \hhline{~----}
        & MNIST & SVHN & CIFAR10 & USPS \\
        \hline\hline
         \textit{Attributes} & \multicolumn{4}{c|}{} \\
        \hline
         Classes & $10$ & $10$ & $10$ & $10$ \\
         \hline
         Image size & $28$ & $32$ & $32$ & $16$ \\
         \hline
         Channels & $1$ & $3$ & $3$ & $1$ \\
         \hline
         Patch size & $7$ & $8$ & $8$ & $4$ \\
         \hline
    \end{tabular}
    \label{tab:Data_details}
\end{table}

\subsection{ViT — MNIST, SVHN, CIFAR10, CIFAR100} \label{app:ViT_details}
\begin{table}[h!]
    \centering
    \caption{Summary of model architecture and training parameters for the ViT model used per dataset.}
    \begin{tabular}{|c|c|c|c|}
        \hline
        \multirow{2}{*}{\textbf{ViT}} & \multicolumn{3}{c|}{Dataset} \\
         \hhline{~---}
        & MNIST & SVHN & CIFAR10 \\
        \hline\hline
         \textit{Model architecture} & \multicolumn{3}{c|}{} \\
         \hline
         Embedding size & $64$ & $64$  & $64$ \\
         \hline
         Attention heads & $8$ & $8$ & $8$ \\
         \hline
         MLP dimension & $256$ & $256$ & $256$ \\
         \hline\hline
         Training & \multicolumn{3}{c|}{} \\
         \hline
         Optimizer & Adam & Adam & Adam \\
         \hline
         Learning rate & $1e-4$ & $1e-4$ & $1e-4$ \\
         \hline
         Batch size & $128$ & $128$ & $128$ \\
         \hline
         Epochs & $20$ & $100$ & $200$ \\
         \hline
         Data augmentation & None & None & None \\
         \hline\hline
         Best model metrics & \multicolumn{3}{c|}{} \\
         \hline
         Top-1 accuracy & $97.57\%$  & $90.94\%$ & $75.47\%$ \\
         \hline
    \end{tabular}
    \label{tab:ViT_details}
\end{table}

\subsection{Swin — SVHN, CIFAR10} \label{app:Swin_details}

\begin{table}[h!]
    \centering
    \caption{Summary of model architecture and training parameters for the Swin transformer model used per dataset.}
    \begin{tabular}{|c|c|c|}
        \hline
        \multirow{2}{*}{\textbf{Swin}} & \multicolumn{2}{c|}{Dataset} \\
         \hhline{~--}
        & SVHN & CIFAR10 \\
        \hline\hline
         \textit{Model architecture} & \multicolumn{2}{c|}{} \\
         \hline
         Embedding size & $64$  & $64$ \\
         \hline
         Attention heads & $4$ & $4$ \\
         \hline
         Window size & $4$ & $4$ \\
         \hline
         MLP dimension & $256$ & $256$ \\
         \hline\hline
         \textit{Training} & \multicolumn{2}{c|}{} \\
         \hline
         Optimizer & Adam & Adam \\
         \hline
         Learning rate & $1e-4$ & $1e-4$ \\
         \hline
         Batch size & $128$ & $128$ \\
         \hline
         Dropout rate & $0$ & $0$ \\
         \hline
         Att. dropout rate & $0$ & $0$ \\
         \hline
         Stochastic depth rate & $0.1$ & $0.1$ \\
         \hline
         Batch size & $128$ & $128$ \\
         \hline
         Epochs & $100$ & $100$ \\
         \hline
         Data augmentation & None & None \\
         \hline\hline
         Best model metrics & \multicolumn{2}{c|}{} \\
         \hline
         Top-1 accuracy & $90.81\%$ & $73.23\%$ \\
         \hline
    \end{tabular}
    \label{tab:Swin_details}
\end{table}

\newpage
\section{Additional experiment results}

\subsection{Ablation study}\label{app:ablation_extra}

\begin{figure}[h!]
    \centering
    \includegraphics[width=\linewidth]{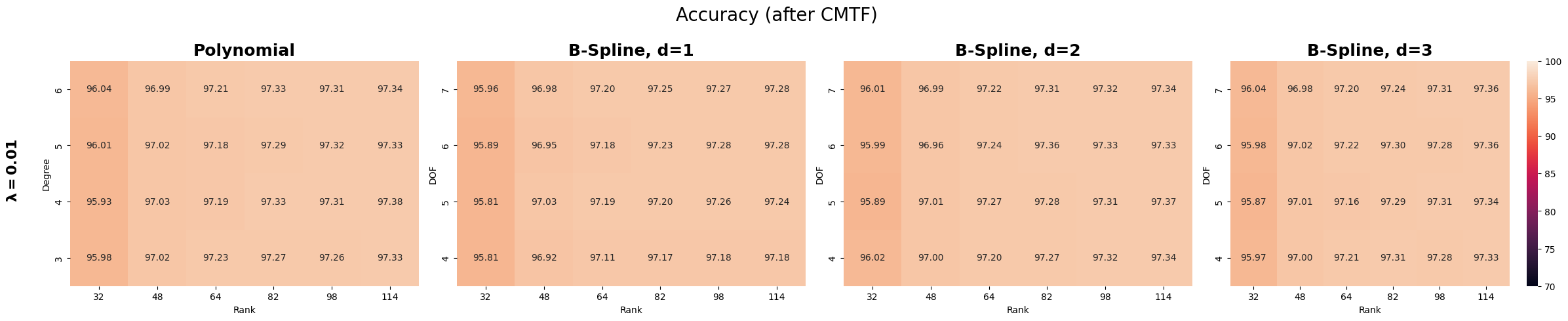}
    \caption{Mean top-$1$ accuracy results for the R-CMTF-BSD and R-CMTF-PD algorithms over $5$ runs for different hyperparameter configurations. The figure layout and axis are as mentioned in the caption of Figure \ref{fig:ablation_acc}, with only one row for $\lambda=0.01$.}
    \label{fig:ablation_Acc_App}
\end{figure}

\begin{figure}[h!]
    \centering
    \includegraphics[width=\linewidth]{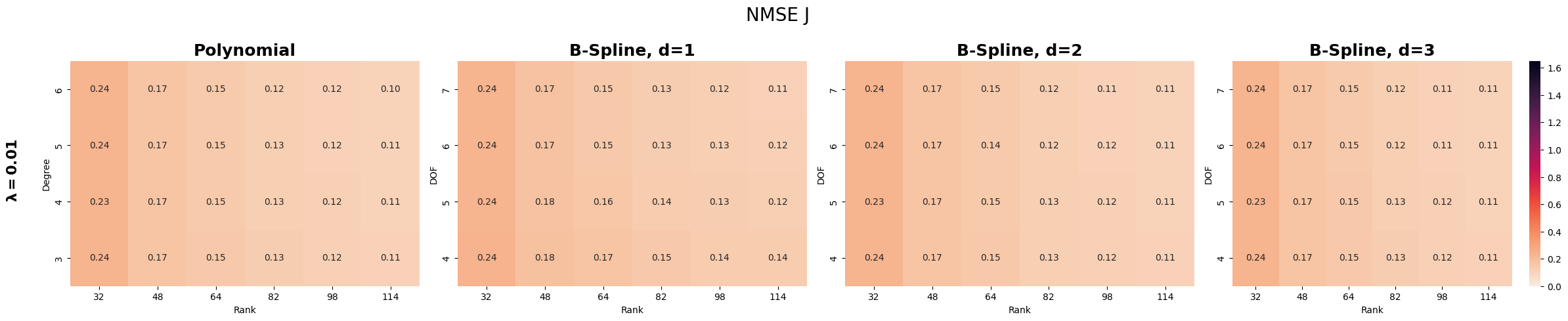}
    \caption{Mean Error($\mathcal{J}$) results for the R-CMTF-BSD and R-CMTF-PD algorithms over $5$ runs for different hyperparameter configurations. The figure layout and axis are as mentioned in the caption of Figure \ref{fig:ablation_acc}, with only one row for $\lambda=0.01$.}
    \label{fig:ablation_J_App}
\end{figure}

\begin{figure}[h!]
    \centering
    \includegraphics[width=\linewidth]{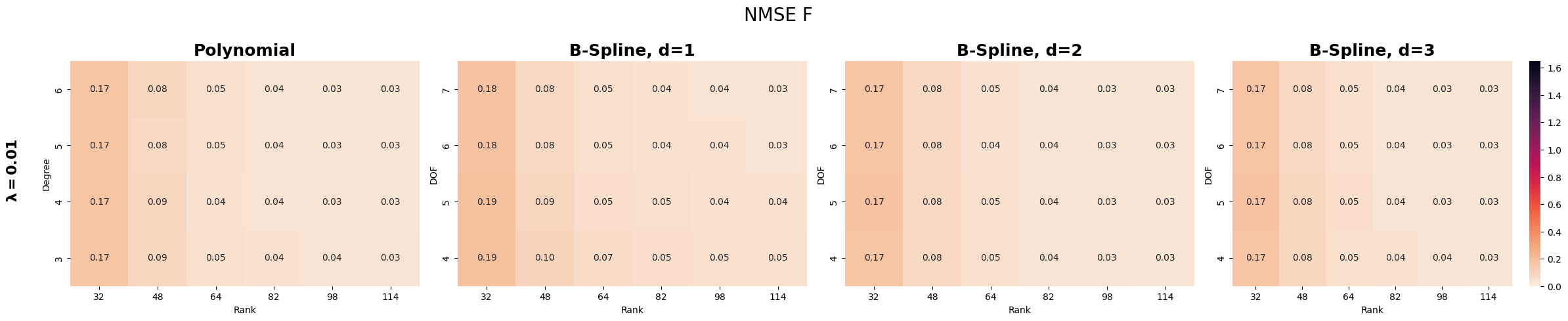}
    \caption{Mean Error($\mathbf{F}$) results for the R-CMTF-BSD and R-CMTF-PD algorithms over $5$ runs for different hyperparameter configurations. The figure layout and axis are as mentioned in the caption of Figure \ref{fig:ablation_acc}, with only one row for $\lambda=0.01$.}
    \label{fig:ablation_F_App}
\end{figure}

\newpage

\subsection{Full compression procedure}\label{app:full_extra}

\begin{table}[h!]
    \centering
    \caption{Mean top-$1$ accuracy results and standard deviation for each step of the front-to-back (FB) and back-to-front (BF) procedures, outlined in Subsection \ref{sec:procedure_compression}, applied to the ViT model trained on SVHN. For each row, the best accuracy result is in bold.}
    \begin{tabular}{|C{1.5cm}|C{1.5cm}|C{2.5cm}|C{2.5cm}|C{2.5cm}|C{2.5cm}|}
        \hline
        \multicolumn{6}{|c|}{\textbf{ViT + SVHN}} \\
        \hline\hline
         \multicolumn{2}{|c|}{\multirow{2}{*}{Front to back}} & Poly & \multicolumn{3}{c|}{B-spline, $\nu = 4$}  \\
         \hhline{~~----}
          \multicolumn{2}{|c|}{} & $d=3$ & $d=1$ & $d=2$ & $d=3$ \\
         \hline\hline
         \multicolumn{2}{|c|}{CMTF (1)} & $80.73 \pm 0.698$ & $\mathbf{84.09 \pm 0.388}$ & $81.24 \pm 0.351$ & $79.97 \pm 0.163$ \\
         \hhline{~-----}
         \multicolumn{2}{|c|}{FT (1)} & $89.19 \pm 0.071$ & $\mathbf{89.56 \pm 0.163}$ & $89.22 \pm 0.203$ & $89.16 \pm 0.234$ \\
         \hline\hline
         \multicolumn{2}{|c|}{CMTF (2)} & $67.73 \pm 1.577$ & $\mathbf{75.7 \pm 1.564}$ & $69.32 \pm 1.218$ & $67.6 \pm 1.295$ \\
         \hhline{~-----}
         \multicolumn{2}{|c|}{FT (2)} & $86.56 \pm 0.252$ & $\mathbf{87.8 \pm 0.592}$ & $86.63 \pm 0.073$ & $86.32 \pm 0.09$ \\
         \hline\hline
         \multicolumn{2}{|c|}{CMTF (3)} & $74.44 \pm 0.613$ & $\mathbf{79.16 \pm 1.567}$ & $73.36 \pm 1.542$ & $74.39 \pm 0.278$ \\
         \hhline{~-----}
         \multicolumn{2}{|c|}{FT (3)} & $85.22 \pm 0.441$ & $\mathbf{86.75 \pm 0.241}$ & $85.59 \pm 0.384$ & $84.84 \pm 0.228$ \\
        \hline\hline
         \multicolumn{2}{|c|}{\multirow{2}{*}{Back to front}} & \multicolumn{4}{c|}{\multirow{2}{*}{}}  \\
         \hhline{~~~~~}
          \multicolumn{2}{|c|}{} & \multicolumn{4}{c|}{}  \\
         \hline\hline
         \multicolumn{2}{|c|}{CMTF (1)} & $90.55 \pm 0.065$ & $90.57 \pm 0.131$ & $\mathbf{90.69 \pm 0.089}$ & $90.58 \pm 0.028$  \\
         \hhline{~-----}
         \multicolumn{2}{|c|}{FT (1)} & $90.8 \pm 0.06$ & $90.87 \pm 0.087$ & $90.85 \pm 0.033$ & $\mathbf{90.89 \pm 0.041}$ \\
         \hline\hline
         \multicolumn{2}{|c|}{CMTF (2)} & $90.12 \pm 0.075$ & $90.12 \pm 0.219$ & $90.12 \pm 0.051$ & $\mathbf{90.22 \pm 0.045}$ \\
         \hhline{~-----}
         \multicolumn{2}{|c|}{FT (2)} & $90.65 \pm 0.125$ & $\mathbf{90.78 \pm 0.07}$ & $90.71 \pm 0.03$ & $90.77 \pm 0.076$ \\
         \hline\hline
         \multicolumn{2}{|c|}{CMTF (3)} & $88.11 \pm 0.221$ & $88.13 \pm 0.356$ & $\mathbf{88.87 \pm 0.13}$ & $88.75 \pm 0.347$ \\
         \hhline{~-----}
         \multicolumn{2}{|c|}{FT (3)} & $90.38 \pm 0.118$ & $\mathbf{90.53 \pm 0.074}$ & $90.48 \pm 0.038$ & $90.51 \pm 0.012$ \\
         \hline
    \end{tabular}
    \label{tab:placeholder1}
\end{table}

\begin{table}[h!]
    \centering
    \caption{Mean top-$1$ accuracy results and standard deviation for each step of the front-to-back (FB) and back-to-front (BF) procedures, outlined in Subsection \ref{sec:procedure_compression}, applied to the ViT model trained on CIFAR-10. For each row, the best accuracy result is in bold.}
    \begin{tabular}{|C{1.5cm}|C{1.5cm}|C{2.5cm}|C{2.5cm}|C{2.5cm}|C{2.5cm}|}
        \hline
        \multicolumn{6}{|c|}{\textbf{ViT + CIFAR-10}} \\
        \hline\hline
         \multicolumn{2}{|c|}{\multirow{2}{*}{Front to back}} & Poly & \multicolumn{3}{c|}{B-spline, $\nu = 4$}  \\
         \hhline{~~----}
          \multicolumn{2}{|c|}{} & $d=3$ & $d=1$ & $d=2$ & $d=3$ \\
         \hline\hline
         \multicolumn{2}{|c|}{CMTF (1)} & $69.8 \pm 0.508$ & $\mathbf{70.14 \pm 0.141}$ & $70.12 \pm 0.307$ & $69.94 \pm 0.287$ \\
         \hhline{~-----}
         \multicolumn{2}{|c|}{FT (1)} & $73.85 \pm 0.084$ & $73.86 \pm 0.199$ & $\mathbf{73.97 \pm 0.028}$ & $73.88 \pm 0.111$ \\
         \hline\hline
         \multicolumn{2}{|c|}{CMTF (2)} & $50.24 \pm 0.664$ & $\mathbf{55.44 \pm 1.823}$ & $49.15 \pm 1.735$ & $50.56 \pm 0.778$ \\
         \hhline{~-----}
         \multicolumn{2}{|c|}{FT (2)} & $71.03 \pm 0.216$ & $\mathbf{72.18 \pm 0.17}$ & $70.97 \pm 0.257$ & $71.15 \pm 0.172$ \\
         \hline\hline
         \multicolumn{2}{|c|}{CMTF (3)} & $66.27 \pm 0.404$ & $\mathbf{68.13 \pm 0.389}$ & $65.99 \pm 0.155$ & $66.01 \pm 0.555$ \\
         \hhline{~-----}
         \multicolumn{2}{|c|}{FT (3)} & $70.16 \pm 0.082$ & $\mathbf{71.6 \pm 0.253}$ & $70.24 \pm 0.243$ & $70.17 \pm 0.165$ \\
        \hline\hline
         \multicolumn{2}{|c|}{\multirow{2}{*}{Back to front}} & \multicolumn{4}{c|}{\multirow{2}{*}{}}  \\
         \hhline{~~~~~}
          \multicolumn{2}{|c|}{} & \multicolumn{4}{c|}{}  \\
         \hline\hline
         \multicolumn{2}{|c|}{CMTF (1)} & $74.25 \pm 0.151$ & $74.2 \pm 0.275$ & $\mathbf{74.43 \pm 0.052}$ & $74.19 \pm 0.033$  \\
         \hhline{~-----}
         \multicolumn{2}{|c|}{FT (1)} & $74.76 \pm 0.125$ & $\mathbf{74.87 \pm 0.057}$ & $74.7 \pm 0.114$ & $74.71 \pm 0.043$ \\
         \hline\hline
         \multicolumn{2}{|c|}{CMTF (2)} & $\mathbf{72.57 \pm 0.176}$ & $72.28 \pm 0.314$ & $72.29 \pm 0.128$ & $71.96 \pm 0.495$ \\
         \hhline{~-----}
         \multicolumn{2}{|c|}{FT (2)} & $\mathbf{74.51 \pm 0.16}$ & $74.44 \pm 0.176$ & $74.33 \pm 0.249$ & $74.43 \pm 0.201$ \\
         \hline\hline
         \multicolumn{2}{|c|}{CMTF (3)} & $71.51 \pm 0.173$ & $71.32 \pm 0.228$ & $71.57 \pm 0.293$ & $\mathbf{71.78 \pm 0.063}$ \\
         \hhline{~-----}
         \multicolumn{2}{|c|}{FT (3)} & $73.61 \pm 0.205$ & $73.59 \pm 0.078$ & $\mathbf{73.63 \pm 0.148}$ & $73.6 \pm 0.091$  \\
         \hline
    \end{tabular}
    \label{tab:placeholder2}
\end{table}

\begin{table}[h!]
    \centering
    \caption{Mean top-$1$ accuracy results and standard deviation for each step of the front-to-back (FB) and back-to-front (BF) procedures, outlined in Subsection \ref{sec:procedure_compression}, applied to the Swin model trained on SVHN. For each row, the best accuracy result is in bold.}
    \begin{tabular}{|C{1.5cm}|C{1.5cm}|C{2.5cm}|C{2.5cm}|C{2.5cm}|C{2.5cm}|}
        \hline
        \multicolumn{6}{|c|}{\textbf{Swin + SVHN}} \\
        \hline\hline
         \multicolumn{2}{|c|}{\multirow{2}{*}{Front to back}} & Poly & \multicolumn{3}{c|}{B-spline, $\nu = 4$}  \\
         \hhline{~~----}
          \multicolumn{2}{|c|}{} & $d=3$ & $d=1$ & $d=2$ & $d=3$ \\
         \hline\hline
         \multicolumn{2}{|c|}{CMTF (1)} & $87.97 \pm 0.226$ & $87.04 \pm 0.36$ & $87.99 \pm 0.226$ & $\mathbf{88.0 \pm 0.18}$ \\
         \hhline{~-----}
         \multicolumn{2}{|c|}{FT (1)} & $88.9 \pm 0.122$ & $88.64 \pm 0.109$ & $\mathbf{88.94 \pm 0.082}$ & $88.91 \pm 0.05$ \\
         \hline\hline
         \multicolumn{2}{|c|}{CMTF (2)} & $85.04 \pm 0.159$ & $83.68 \pm 0.102$ & $\mathbf{85.23 \pm 0.107}$ & $84.77 \pm 0.399$ \\
         \hhline{~-----}
         \multicolumn{2}{|c|}{FT (2)} & $\mathbf{87.66 \pm 0.229}$ & $87.28 \pm 0.093$ & $87.62 \pm 0.284$ & $87.59 \pm 0.278$ \\
         \hline\hline
         \multicolumn{2}{|c|}{CMTF (3)} & $84.42 \pm 0.135$ & $82.19 \pm 0.842$ & $\mathbf{84.56 \pm 0.109}$ & $84.05 \pm 0.38$  \\
         \hhline{~-----}
         \multicolumn{2}{|c|}{FT (3)} & $86.87 \pm 0.195$ & $85.99 \pm 0.077$ & $\mathbf{86.91 \pm 0.386}$ & $86.64 \pm 0.193$ \\
        \hline\hline
         \multicolumn{2}{|c|}{\multirow{2}{*}{Back to front}} & \multicolumn{4}{c|}{\multirow{2}{*}{}}  \\
         \hhline{~~~~~}
          \multicolumn{2}{|c|}{} & \multicolumn{4}{c|}{}  \\
         \hline\hline
         \multicolumn{2}{|c|}{CMTF (1)} & $90.77 \pm 0.009$ & $90.74 \pm 0.028$ & $90.72 \pm 0.033$ & $\mathbf{90.8 \pm 0.06}$ \\
         \hhline{~-----}
         \multicolumn{2}{|c|}{FT (1)} & $\mathbf{90.23 \pm 0.117}$ & $90.18 \pm 0.135$ & $90.22 \pm 0.103$ & $90.22 \pm 0.107$ \\
         \hline\hline
         \multicolumn{2}{|c|}{CMTF (2)} & $89.77 \pm 0.067$ & $89.47 \pm 0.288$ & $\mathbf{89.92 \pm 0.085}$ & $89.72 \pm 0.066$ \\
         \hhline{~-----}
         \multicolumn{2}{|c|}{FT (2)} & $90.06 \pm 0.074$ & $90.12 \pm 0.123$ & $\mathbf{90.19 \pm 0.119}$ & $90.16 \pm 0.083$ \\
         \hline\hline
         \multicolumn{2}{|c|}{CMTF (3)} & $88.2 \pm 0.213$ & $87.71 \pm 0.521$ & $\mathbf{88.65 \pm 0.275}$ & $88.39 \pm 0.428$ \\
         \hhline{~-----}
         \multicolumn{2}{|c|}{FT (3)} & $89.49 \pm 0.142$ & $89.44 \pm 0.151$ & $\mathbf{89.63 \pm 0.05}$ & $89.53 \pm 0.25$  \\
         \hline
    \end{tabular}
    \label{tab:placeholder3}
\end{table}

\begin{table}[h!]
    \centering
    \caption{Mean top-$1$ accuracy results and standard deviation for each step of the front-to-back (FB) and back-to-front (BF) procedures, outlined in Subsection \ref{sec:procedure_compression}, applied to the Swin model trained on CIFAR-10. For each row, the best accuracy result is in bold.}
    \begin{tabular}{|C{1.5cm}|C{1.5cm}|C{2.5cm}|C{2.5cm}|C{2.5cm}|C{2.5cm}|}
        \hline
        \multicolumn{6}{|c|}{\textbf{Swin + CIFAR-10}} \\
        \hline\hline
         \multicolumn{2}{|c|}{\multirow{2}{*}{Front to back}} & Poly & \multicolumn{3}{c|}{B-spline, $\nu = 4$}  \\
         \hhline{~~----}
          \multicolumn{2}{|c|}{} & $d=3$ & $d=1$ & $d=2$ & $d=3$ \\
         \hline\hline
         \multicolumn{2}{|c|}{CMTF (1)} & $69.8 \pm 0.545$ & $69.28 \pm 0.221$ & $\mathbf{70.56 \pm 0.227}$ & $70.39 \pm 0.387$ \\
         \hhline{~-----}
         \multicolumn{2}{|c|}{FT (1)} & $71.78 \pm 0.119$ & $71.73 \pm 0.151$ & $\mathbf{71.86 \pm 0.192}$ & $71.65 \pm 0.101$ \\
         \hline\hline
         \multicolumn{2}{|c|}{CMTF (2)} & $67.03 \pm 0.253$ & $66.34 \pm 0.124$ & $66.66 \pm 0.405$ & $\mathbf{67.19 \pm 0.364}$ \\
         \hhline{~-----}
         \multicolumn{2}{|c|}{FT (2)} & $70.6 \pm 0.212$ & $70.47 \pm 0.082$ & $\mathbf{70.62 \pm 0.394}$ & $70.6 \pm 0.046$ \\
         \hline\hline
         \multicolumn{2}{|c|}{CMTF (3)} & $68.56 \pm 0.298$ & $67.94 \pm 0.171$ & $68.7 \pm 0.215$ & $\mathbf{68.81 \pm 0.184}$ \\
         \hhline{~-----}
         \multicolumn{2}{|c|}{FT (3)} & $70.07 \pm 0.3$ & $69.79 \pm 0.446$ & $\mathbf{70.1 \pm 0.255}$ & $69.97 \pm 0.24$ \\
        \hline\hline
         \multicolumn{2}{|c|}{\multirow{2}{*}{Back to front}} & \multicolumn{4}{c|}{\multirow{2}{*}{}}  \\
         \hhline{~~~~~}
          \multicolumn{2}{|c|}{} & \multicolumn{4}{c|}{}  \\
         \hline\hline
         \multicolumn{2}{|c|}{CMTF (1)} & $73.11 \pm 0.159$ & $73.14 \pm 0.028$ & $73.08 \pm 0.057$ & $\mathbf{73.21 \pm 0.131}$ \\
         \hhline{~-----}
         \multicolumn{2}{|c|}{FT (1)} & $72.5 \pm 0.123$ & $72.45 \pm 0.073$ & $72.53 \pm 0.045$ & $\mathbf{72.56 \pm 0.064}$ \\
         \hline\hline
         \multicolumn{2}{|c|}{CMTF (2)} & $71.94 \pm 0.17$ & $71.91 \pm 0.177$ & $71.91 \pm 0.353$ & $\mathbf{72.12 \pm 0.367}$ \\
         \hhline{~-----}
         \multicolumn{2}{|c|}{FT (2)} & $72.34 \pm 0.116$ & $72.36 \pm 0.176$ & $\mathbf{72.48 \pm 0.045}$ & $72.33 \pm 0.148$ \\
         \hline\hline
         \multicolumn{2}{|c|}{CMTF (3)} & $70.73 \pm 0.201$ & $70.72 \pm 0.131$ & $70.85 \pm 0.347$ & $\mathbf{71.13 \pm 0.266}$ \\
         \hhline{~-----}
         \multicolumn{2}{|c|}{FT (3)} & $72.0 \pm 0.174$ & $71.69 \pm 0.165$ & $72.01 \pm 0.202$ & $\mathbf{72.05 \pm 0.036}$ \\
         \hline
    \end{tabular}
    \label{tab:placeholder4}
\end{table}

\end{document}